\documentclass{article}

\usepackage{microtype}
\usepackage{graphicx}
\usepackage{subfigure}
\usepackage{booktabs} 

\usepackage{newtxtext,newtxmath}
\usepackage{amsmath}
\usepackage{breqn}
\usepackage{enumitem}
\newcommand{\set}[1]{\ensuremath{\mathcal{#1}}}
\newcommand{\relset}[0]{{\set{R}}}
\newcommand{\entset}[0]{{\set{E}}}
\newcommand{\score}[0]{{\phi_\theta}}
\newcommand{\scorecirc}[0]{{\phi^{\circ}_\theta}}
\newcommand{\realset}[0]{\ensuremath{\mathbb{R}}}
\newcommand{\vc}[1]{\ensuremath{\mathbf{#1}}}

\newcommand{\acr}[1]{{{\textsc{#1}}}}
\newcommand{\ourmodel}[0]{ReAlE}
\newcommand{\synthdata}{\mbox{\acr{REL-ER}}}
\newcommand{\maxar}[0]{\alpha}
\newcommand{\rv}[0]{\vc{r}}

\newcommand{\eg}[0]{\emph{e.g.},~}
\newcommand{\rank}[0]{\func{rank}}
\newcommand{\func}[1]{\ensuremath{\mathrm{#1}}}
\newcommand{\scoreH}[0]{{\phi_\theta^H}}
\usepackage{booktabs}
\usepackage{mymathfont}
\newtheorem{theorem}{Theorem}

\newcommand{\op}[1]{{\fontfamily{lmss}\selectfont{#1}}}

\usepackage{hyperref}

\usepackage[ruled, lined, longend]{algorithm2e}
\newtheorem{innercustomlemma}{Lemma}
\newenvironment{customlemma}[1]
  {\renewcommand\theinnercustomlemma{#1}\innercustomlemma}
  {\endinnercustomlemma}
  
\newenvironment{proof}{\noindent\textbf{Proof: }\ignorespaces}
  {\hspace*{\fill}$\Box$\medskip}
  
\usepackage[accepted]{icml2021}

\icmltitlerunning{Knowledge Hypergraph Embedding Meets Relational Algebra}

\begin{document}

\twocolumn[
\icmltitle{Knowledge Hypergraph Embedding Meets Relational Algebra}



\icmlsetsymbol{equal}{*}

\begin{icmlauthorlist}
\icmlauthor{Bahare Fatemi}{ubc,equal}
\icmlauthor{Perouz Taslakian}{eai}
\icmlauthor{David Vazquez}{eai}
\icmlauthor{David Poole}{ubc}

\end{icmlauthorlist}

\icmlaffiliation{ubc}{University of British Columbia}
\icmlaffiliation{eai}{Element AI}

\icmlcorrespondingauthor{Bahare Fatemi}{bfatemi@cs.ubc.ca}

\icmlkeywords{Machine Learning, ICML}

\vskip 0.3in
]



\printAffiliationsAndNotice{\icmlEqualContribution} 

\begin{abstract}
Embedding-based methods for reasoning in knowledge hypergraphs learn a representation for each entity and relation. Current methods do not capture the procedural rules underlying the relations in the graph. We propose a simple embedding-based model called \ourmodel{} that performs link prediction in \emph{knowledge hypergraphs} (generalized knowledge graphs) and can represent high-level abstractions in terms of relational algebra operations. We show theoretically that \ourmodel{} is fully expressive and provide proofs and empirical evidence that it can represent a large subset of the primitive relational algebra operations, namely \op{renaming}, \op{projection}, \op{set union}, \op{selection}, and \op{set difference}. We also verify experimentally that \ourmodel{} outperforms state-of-the-art models in knowledge hypergraph completion, and in representing each of these primitive relational algebra operations. For the latter experiment, we generate a synthetic knowledge hypergraph, for which we design an algorithm based on the Erd\H{o}s-R\'enyi model for generating random graphs.
\end{abstract}


\section{Introduction}
\emph{Knowledge hypergraphs} are knowledge bases that store information about the world in the form of tuples describing relations among entities. They can be seen as a generalization of knowledge graphs in which relations are defined on two entities.
In recent times, the most dominant paradigm has been to learn and reason on knowledge graphs. However, hypergraphs are the original structure of many existing graph datasets. \citet{m-TransH} observe that in the original \acr{Freebase}~\cite{bollacker2008freebase} more than $1/3$rd of the entities participate in non-binary relations (i.e., defined on more than two entities). \citet{HypE} observe, in addition, that $61$\% of the relations in the original \acr{Freebase} are non-binary.
Similar to knowledge graphs, knowledge hypergraphs are incomplete because curating and storing all the true information in the world is difficult. The goal of \emph{link prediction} in knowledge hypergraphs (or \emph{knowledge hypergraph completion}) is to predict unknown links or relationships among entities based on existing ones. While it is possible to convert a knowledge hypergraph into a knowledge graph and apply existing methods on it, \citet{m-TransH} and \citet{HypE} show that embedding-based methods for knowledge graph completion do not work well out of the box for knowledge graphs obtained through such conversion techniques.

Most of the models developed to perform link prediction in knowledge hypergraphs are extensions of those used for link prediction in knowledge graphs. But a model designed to reason over binary relations does not necessarily generalize well to non-binary relations. Furthermore, it is not clear if the existing models for knowledge hypergraph completion effectively capture the relational semantics underlying the knowledge hypergraphs.
Recent research~\cite{battaglia2018relational,teru2020inductive} has highlighted the role of relational inductive biases in building learning agents that learn entity-independent relational semantics and reason in a compositional manner. Such agents can generalize better to unseen structures. In this work, we are interested in exploring the foundations of the knowledge hypergraph completion task; we thus aim to design a model for reasoning in knowledge hypergraphs that is simple, expressive, and can represent high-level abstractions in terms of relational algebra operations. 
We hypothesize that models that can reason about relations in terms of relational algebra operations have better generalization power.

\emph{Relational algebra} is a formalization that is at the core of relational models (e.g., relational databases). 
It consists of several \emph{primitive operations} that can be combined to synthesize all other operations used.
Each such operation takes relations as input and returns a relation as output.
In a knowledge hypergraph, these operations can describe how relations depend on each other.
To illustrate the connection between relational algebra operations and relations in knowledge hypergraphs, consider the example in Figure~\ref{fig:relational-algebra} that shows samples from some knowledge hypergraph.
The relations in the example feature the two primitive relational algebra operations \op{renaming} and \op{projection}. 
Relation $bought$ is a \op{renaming} of $sold$.
Relation $buyer$ is a \op{projection} of relation $sold$.
If a model is able to represent 
these two operations, it can potentially learn that 
a tuple $bought(X, Y, I)$ (person $X$ bought from person $Y$ item $I$) is implied by tuple $sold(Y, X, I)$ (person $Y$ sold to person $X$ item $I$);
or that 
a tuple $buyer(X, I)$ (person $X$ is the buyer of item $I$) is implied by the tuple $sold(Y, X, I)$. 
An embedding model that cannot represent the relational algebra 
operations \op{renaming} and \op{projection} would not be able to learn that relation $bought$ in Figure~\ref{fig:relational-algebra} is a \op{renaming} of relation $sold$. It would thus be difficult for such a model to reason about the relationship between these two relations.
In contrast, a model that can represent \op{renaming} and \op{projection} operations 
is potentially able to determine that $bought(alex, drew, book)$ is true because the train set contains $sold(drew, alex, book)$ and $bought$ is a \op{renaming} of $sold$. 

Designing reasoning methods that can capture the relational semantics in terms of relational algebra is especially important in the context of knowledge hypergraphs,
where many relations are defined on more than two entities.
Domains with beyond-binary relations often provide multiple methods of expressing the same underlying notion, as seen in the above example where relations $sold$ and $bought$ encode the same information. 
In contrast, since all relations in a knowledge graph are binary (have arity $2$), many of the relational algebra operations (such as \op{projection}, which changes the arity of the relation) are not applicable to this setting. 
This also highlights the importance of knowledge hypergraphs as a data model that encodes rich relational structures ripe for further exploration.
\begin{figure}[t]
    \begin{center}
    \includegraphics[width=0.45\textwidth]{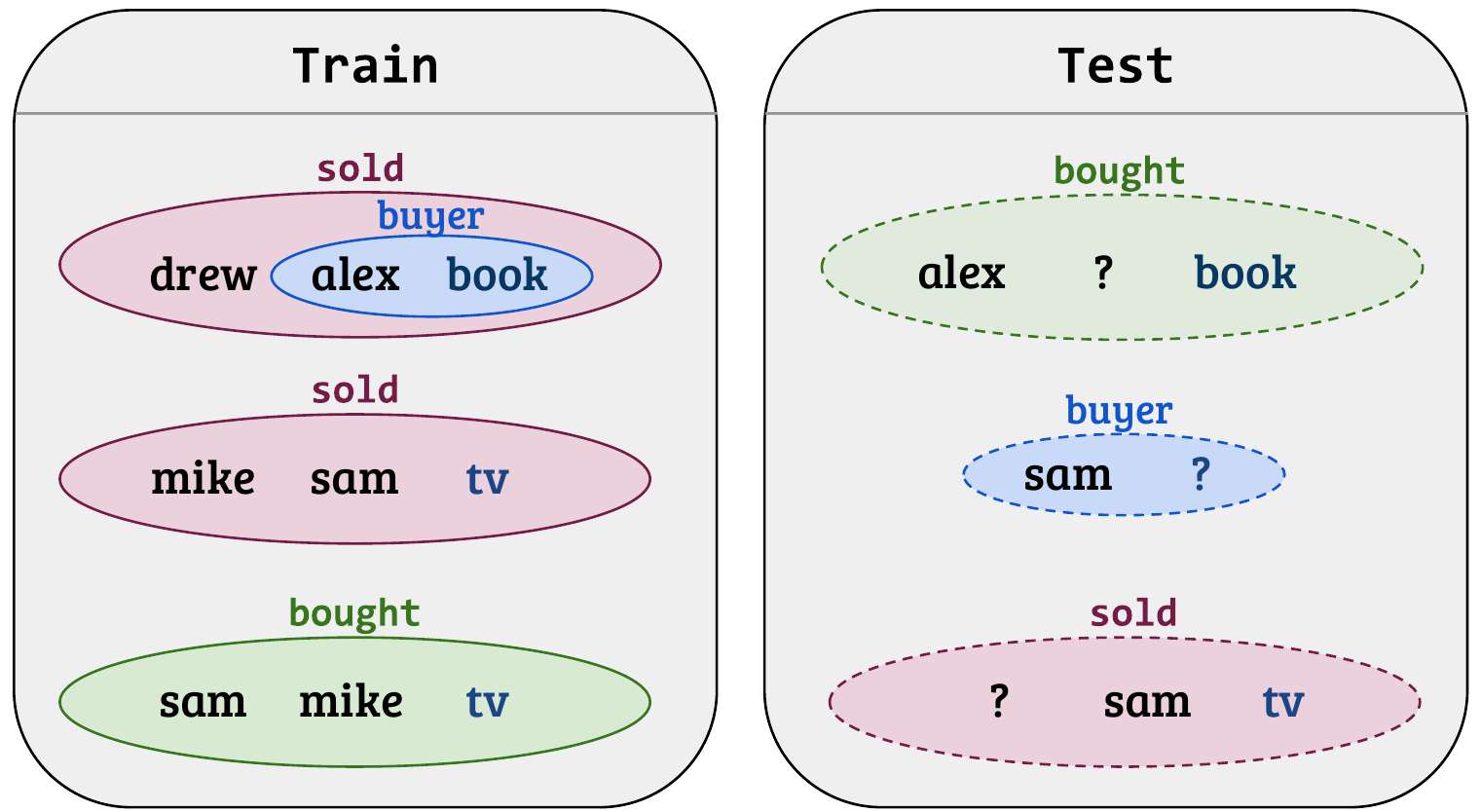}
  \end{center}
  \caption{An example of a knowledge hypergraph split into train and test sets. The train set contains tuples $sold(drew, alex, book)$, $buyer(alex, book)$, $sold(mike, sam, tv)$, and $bought(sam, mike, tv)$. 
  Relation $bought$ can be obtained by applying a \op{renaming} operation to relation $sold$. Similarly, relation $buyer$ is a \op{projection} of relation $sold$. 
  Learning these relational algebra operations can help the model generalize better to the tuples in the test set. The test responses, from top to bottom, are $drew$, $tv$, and $mike$.
  }
  \label{fig:relational-algebra}
\end{figure}
Our goal in this work is to design a knowledge hypergraph model that is simple, expressive, and can provably represent relational algebra operations. To achieve this, we propose a model called \ourmodel{}.
We show theoretically that \ourmodel{} can represent the operations \op{renaming}, \op{projection}, \op{selection}, \op{set union}, and \op{set difference}, and experimentally verify that it performs well in predicting such relations in existing knowledge bases.
We further prove that \ourmodel{} is fully expressive.
The contributions of this work are as follows.
\begin{enumerate}[itemsep=0cm,topsep=0.1cm] 
\item \emph{\ourmodel}, an embedding-based method for knowledge hypergraph completion that can provably represent the relational algebra operations \op{renaming}, \op{projection}, \op{set union}, \op{selection}, and \op{set difference},

\item \emph{experimental results} that show that \ourmodel{} outperforms the state-of-the-art on well-known public datasets, and

\item a \emph{framework} for generating synthetic knowledge hypergraphs, 
which is based on the Erd\H{o}s-R\'enyi random graph generation model and can generate relations by repeated application of primitive relational algebra operations. 
Such datasets can serve as benchmarks for evaluating the relational algebraic generalization power of knowledge hypergraph completion methods.
We use our algorithm to generate the dataset \synthdata{} (RELational Erd\H{o}s-R\'enyi) and use it in our experiments. 
\end{enumerate}

\section{Related Work}\label{sec:related-work}
Existing work in the area of knowledge hypergraph completion can be grouped into the following categories.

\textbf{Statistical Relational Learning. } 
Models under the umbrella of statistical relational learning~\cite{raedt2016statistical} can handle variable arity relations and explicitly model the interdependencies of relations. Our work is complementary to these approaches in that ReAlE is an embedding-based model that represents relational algebra operations implicitly, rather than representing them explicitly.

\textbf{Translational approaches. }
Models in this category are extensions of translational approaches for binary relations to beyond-binary. 
One of the earliest models in this category is m-TransH~\cite{m-TransH}, which extends TransH \cite{TransH} to knowledge hypergraph embedding. 
RAE~\cite{RAE} extends m-TransH by adding the \emph{relatedness} of values -- the likelihood that two values co-participate in a common instance -- to the loss function.
NaLP~\cite{NaLP} uses a similar strategy to RAE, but models the relatedness of values based on the roles they play in different tuples.
Models in this category have severe restrictions on the types of relations they can model. \citet{GETD} discuss these limitations, and we address them more formally in Section~\ref{sec:theoretical}.

\textbf{Tensor factorization based approaches. }
Models in this category extend tensor factorization models for knowledge graph completion to knowledge hypergraph completion.
\citet{GETD} propose GETD, which 
extends Tucker~\cite{balavzevic2019tucker} to n-ary relations.
GETD is fully expressive. However, given that for each relation it learns a tensor with a dimension equal to the arity of the relation, 
the memory and time complexity of GETD grow exponentially with the arity of relations in the dataset. 
\citet{HypE} propose HypE which is motivated by SimplE~\cite{kazemi2018simple}.
HypE disentangles the embeddings of relations from the positions of its arguments and thus the memory and time complexity grows linearly with the arity of the relations in the dataset. 
HypE is fully expressive but cannot represent some relational algebra operations. We discuss this in Section~\ref{sec:theoretical}.

\textbf{Key-value pair based approaches.}
Models in this category, such as NeuInfer and HINGE \cite{neuinfer,HINGE}, assume that a tuple is composed of a triple (binary relation), plus a list of attributes in the form of key-value pairs. The problem these works study is slightly different as we consider all information as part of a tuple. 
They evaluate their models on datasets for which they obtain the tuple attributes from external data sources using heuristics and thus
their results are not comparable to ours.

\textbf{Graph neural networks approaches.}
These approaches extend graph neural networks to hypergraph neural networks~\cite{HGNN,HGCN}. G-MPNN~\cite{yadati2020neural} further extends these models to knowledge hypergraphs (directed and labeled hyperedges). The G-MPNN
scoring function assumes relations are symmetric, and thus has restrictions in modeling the non-symmetric relations and is not fully expressive.

\textbf{Present work.} 
Reasoning in hypergraphs is a relatively underexplored area that has recently gained more attention. While the methods above show promise, none of them offer more understanding of the knowledge hypergraph completion task, as these works lack theoretical analysis of the models they propose and do not provide evidence to the type of entity-independent relational semantic they can model.
In this work, we design a model based on relational algebra which is the calculus of relational models. Besides the theoretical contributions, we show empirically how basing our model on relational algebra operations give us improvements compared to the existing works. 

\section{Definition and Notation}
Assume a finite set of entities~$\entset$ and a finite set of relations~$\relset$. Each relation has a fixed non-negative integral \emph{arity}. A \emph{tuple} is the form of $r(x_1, \dots, x_n)$ where $r \in \relset$, $n = |r|$ is the arity of $r$, and each $x_i \in \entset$.
Let $\tau$ be the set of ground truth tuples; it specifies all of the tuples that are true. If a tuple is not in~$\tau$, it is false.
A knowledge hypergraph consists of
a subset of the tuples $\tau' \subseteq \tau$. 
Knowledge hypergraph completion is the problem of predicting the missing tuples in $\tau'$, that is, finding the tuples $\tau \setminus \tau'$.
A knowledge graph is a special case of a knowledge hypergraph where all relations have arity $2$. 

An \emph{embedding} is a function from an entity or a relation to a vector or a matrix (or a higher-order tensor) over a field.
We use bold lower-case for embeddings, that is, $\vc{x}$ is the embedding of entity $x$, and $\rv$ is the embedding of relation $r$.
For the task of knowledge hypergraph completion, an embedding-based model having parameters~$\theta$ defines a function $\phi_\theta$ that takes a tuple as input and generates a prediction, \eg a probability (or a score) of the tuple being true.
A model is \emph{fully expressive} if given any assignment of truth values to all tuples,
there exists an assignment of values to~$\theta$ that accurately separates the true tuples from the false ones.

Following Python notation, $\vc{x}[k]$ is the $k$-th index of vector $\vc{x}$ and   $\vc{r}[i][k]$ is the $i$-th row and $k$-th column of matrix $r$. We use~$\times$ for multiplication of two scalars.
The bijective function $\pi:\{1,\dots,n\} \rightarrow \{1,\dots,n\}$ takes a sequence $(x_1, x_2,\dots, x_n)$ and outputs a sequence $(x_{\pi(1)},x_{\pi(2)},\dots,{x_{\pi(n)}})$. For example, if $\pi=\{1 \mapsto 2, 2 \mapsto 1, 3 \mapsto 3\}$ then,
$(x_{\pi(1)},x_{\pi(2)},x_{\pi(3)}) = (x_2, x_1, x_3)$.


\section{Proposed Method}\label{section:motivation}
\ourmodel{} (Relational Algebra Embedding) is a knowledge hypergraph completion model that has parameters $\theta$ and a scoring function $\score$. We motivate our model bottom-up by first describing an intuitive model for this task (Equation \ref{eq:v1}); 
we then discuss why this formulation would not work well and adjust it in \ourmodel{} (Equation~\ref{eq:main}).

Given a tuple  $r(x_1,\dots,x_n)$, determining whether it is true or false depends on the relation and the entities involved; it also depends
on the position of each entity in the tuple, as the role of an entity changes with its position and relation.
For example, the role of $drew$ is different in the tuples $sold(drew, alex, book)$ and $sold(alex, drew, book)$ as the position of $drew$ is different. The role of $drew$ is also different in the tuples $sold(drew,alex,book)$ and $bought(drew, alex, book)$ as the relation is different.

An intuitive model to decide whether a tuple is true or not is one that 
embeds each entity $x_i \in \entset$ into a vector $\vc{x_i} \in [0,1]^d$ of length~$d$, 
and the relation $r$ into a matrix  $\vc{r} \in \realset^{|r| \times d}$, where the~$i^{th}$ row in $\vc{r}$ operates over the entity at position $i$. 
Each relation $r$ has a bias term $b_r$ as a relation dependent constant that does not depend on any entities and allows the model to have the flexibility to search through the solution space (similar to a bias term in linear regression).
Such a model defines a score for a tuple. This score may be interpreted as ``the higher the score, the more likely it is that the tuple is true'', and can be expressed as follows:

\begin{align}\label{eq:v1}
    \scorecirc{}&(r(x_1, \dots, x_n)) 
    =  \sigma \left(b_{r} + \sum_{i=1}^{|r|} \sum_{k=0}^{d-1} \vc{x_{i}}[k] \times \vc{r}[i][k]\right)
\end{align}
\begin{figure}[t]

\begin{center}
    \includegraphics[width=1.\columnwidth, trim={0 0 0 0 cm},clip]{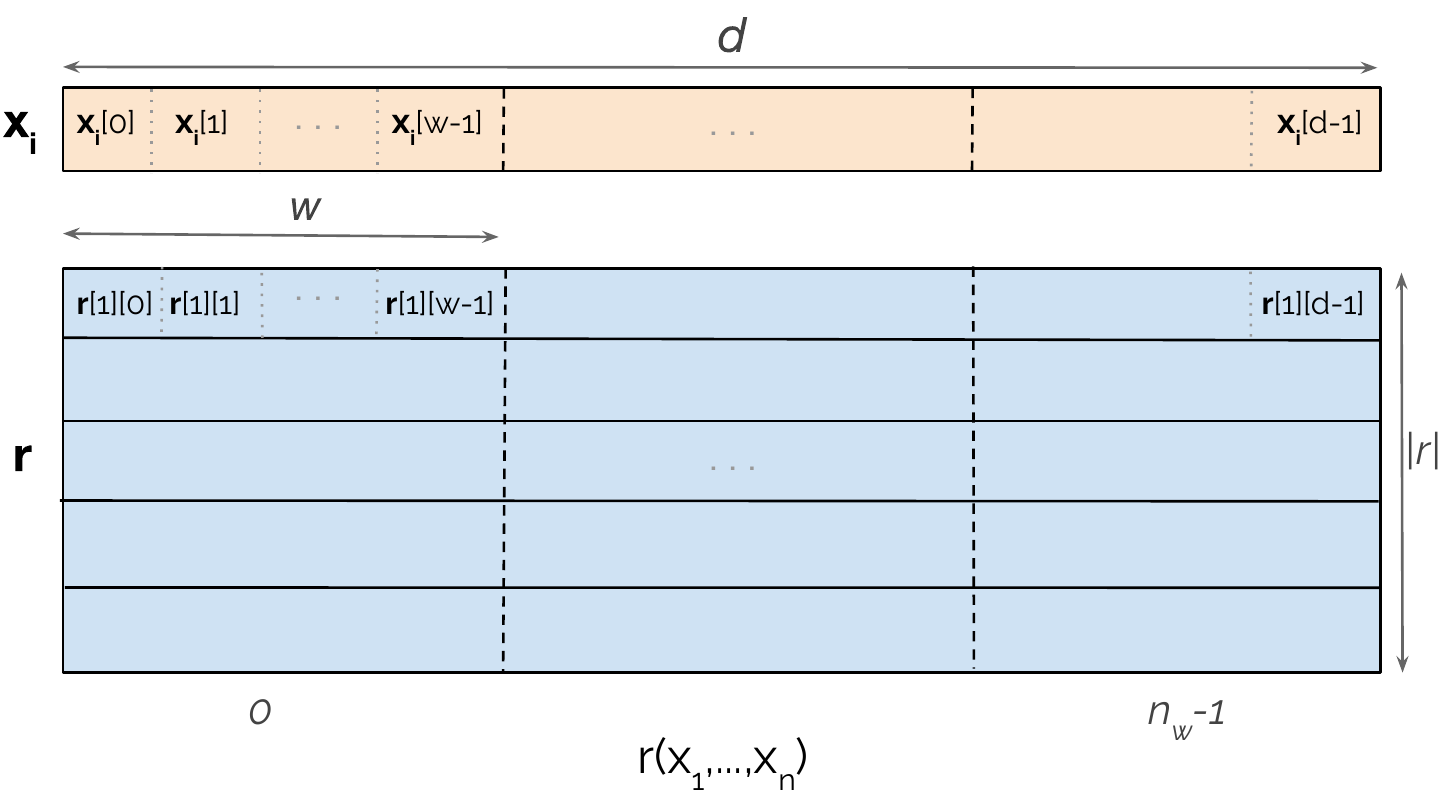}
  \end{center}
  \caption{A schematic of the entity and relation embeddings in ReAlE: the embedding dimension $d$ is divided into $n_w$ windows of size $w$.}
  \label{fig:embeddings}
\end{figure}

The above scoring function is similar to Canonical Polyadic~\cite{CP} for tensor decomposition, as the same indices in the embedding of entities and relations are multiplied and the final score is the result of summing these multiplications.
Studies~\cite{kazemi2018simple,lacroix2018canonical}, however, show that designs similar to that in Equation~\ref{eq:v1} do not properly model the interaction between embeddings of entities and relations at different indices, and that by adding interaction among different elements of the embeddings in different indices, the information flows better and the performance improves.
 
Therefore, one missing component in Equation~\ref{eq:v1} is the interaction of different elements of the embeddings. 
To create such an interaction, 
we introduce the concept of \emph{windows}, a range of indices whereby elements within the same window interact with each other.
\ourmodel{} uses windows to increase interaction between embeddings;
it then uses a nonlinear function $\sigma$ to obtain the contribution of each window, which are then added to produce a score. 
The number of embedding elements for each entity in a window is the \emph{window size} and is a hyperparameter
(see Figure~\ref{fig:embeddings}).
Let $w$ denote the window size, $n_w = \lfloor \frac{d}{w} \rfloor$  the number of windows, and $b_r^{j}$ the bias term of relation $r$ for the $j^{th}$ window, for all $j=0,\dots,n_w-1$.
Equation \ref{eq:main}
defines \ourmodel{}'s score of a tuple $r(x_1, x_2, \dots, x_n)$, where $\sigma$ is a 
monotonically-increasing nonlinear function
that is differentiable almost everywhere. 
\begin{equation}\label{eq:main}
\begin{split}
&\score(r(x_1, \dots, x_n)) = \\ 
&\frac{1}{n_w}\sum_{j=0}^{n_w-1} \sigma\left(b_r^{j} + \sum_{i=1}^{|r|} \sum_{k=0}^{w-1} \vc{x_i}[jw + k] \times \vc{r}[i][jw + k]\right)
\end{split}
\end{equation}

\paragraph{Learning \ourmodel{} Model.}
To learn a \ourmodel{} model, we use stochastic gradient descent with mini-batches. In each learning iteration, we take in a batch of positive tuples from the knowledge hypergraph. As a knowledge hypergraph typically only has positive instances and we need to also train our model on negative instances, we generate negative examples by following the contrastive approach of \citet{TransE}.

Given a knowledge hypergraph defined on $\tau'$, we let $\tau'_{\mathrm{train}}$, $\tau'_{\mathrm{test}}$, 
and~$\tau'_{\mathrm{valid}}$ denote the (pairwise disjoint) train, test, and validation sets, respectively,
so that $\tau' = \tau'_{\mathrm{train}} \sqcup \tau'_{\mathrm{test}} \sqcup \tau'_{\mathrm{valid}}$ where
 $\sqcup$ is disjoint set union.
To build a model that completes $\tau'$, we train it using $\tau'_{\mathrm{train}}$, tune the hyperparameters of the model using $\tau'_{\mathrm{valid}}$, and evaluate its efficacy on $\tau'_{\mathrm{test}}$.
For any tuple $t$ in $\tau'$, we let $T_{\mathrm{neg}}(t)$ be a function that generate a set of related negative samples. We define the following cross entropy loss that has been shown to be effective for link prediction~\citep{baselines-strike}.

\begin{equation}\label{eq:loss}
\mathcal{L}(\theta, \tau'_{\mathrm{batch}}) =  \sum_{t \in \tau'_{\mathrm{batch}}}{-\log\left(\frac{\mathrm{e}^{\score(t)}}{\displaystyle\sum_{t' \in ~\{t\} ~\cup ~ T_{\mathrm{neg}}(t)}{\mathrm{e}^{\score(t')}}}\right)}
\end{equation}

Here, $\theta$ represents parameters of the model including relation and entity embeddings, and $\score$ is the function given by Equation~\eqref{eq:main} that maps a tuple to a score using parameters~$\theta$, and $\tau'_{batch}$ is a batch of tuples in $\tau'_{train}$.
Algorithm~\ref{alg} shows a high-level description of how we train a \ourmodel{} model.

\begin{algorithm}[H]
  \caption{Learning \ourmodel{}}
  \label{alg}
  \SetNoFillComment
  \DontPrintSemicolon
  \KwIn{Tuples $\tau'_{train}$, loss function $\mathcal{L}$, scoring function $\score{}$}
  \KwOut{Embeddings $\vc{x}$ and \vc{r} for all entities and relations in $\tau'_{train}$.}
  Initialize \vc{x} and \vc{r} (at random) \; 
  
  \For {every batch $\tau'_{\mathrm{batch}}$ of tuples in $\tau'_{train}$}{
  
    \For{tuple $t$ in $\tau'_{\mathrm{batch}}$}{
      Generate negative tuples $T_{neg}(t)$ \;
      
      \For {$t' \in \{t\} \cup T_{neg}(t)$}{
        Compute $\score(t')$ ~~~~ (Eq.~\ref{eq:main})\;
      }
      
    } 
    
    Compute the loss $\mathcal{L}(\theta, \tau'_{\mathrm{batch}})$ ~~~~~~~~~~~~ (Eq.~\ref{eq:loss})\;
    
    Compute the gradient of loss with respect to \vc{x} and \vc{r} \;
    
    Update embeddings  \vc{x} and \vc{r} through back-propagation \;
    
  }
\end{algorithm}

\section{Theoretical Analysis}
\label{sec:theoretical}
To better understand the expressive power of \ourmodel{} and the types of reasoning it can perform, in this section we analyze the extent of its expressivity and its capacity to represent relational algebra operations without the operations being known or given to the learner. 

Relational algebra allows for quantification, in particular for statements that are true for all entities or are true for at least one. To allow for such statements, we introduce the notion of \emph{variables}, which, following the convention of Datalog, we write in upper case, e.g., $X_1,\dots,X_n$. We use lower case $x_1, \dots, x_n$ to denote
particular entities. To make a meaningful statement, each variable is quantified using $\forall$ (the statement is true for all assignments of entities to the variable) and $\exists$ (the statement is true if there exists an assignment of an entity to the variable).
$\bar{x}$ is a sequence of particular entities and $\bar{X}$ is a sequence of variables. 
The symbol $\neg$ is the negation operator.

For any $\bar{x}$ and any relation~$r$, we define the  \emph{relation complement} function $f$ as 
$f(\score(r({\bar{x}}))) = \score(\neg r({\bar{x}}))$. 
This function depends on the choice of the nonlinearity~$\sigma$ of the scoring function in Equation~\ref{eq:main}.
For example, if $\sigma$ is the Sigmoid function, then $f(\score(r({\bar{x}}))) = 1 - \score(r({\bar{x})})$; 
if it is the hyperbolic tangent (tanh), then $f(\score(r({\bar{x}}))) = - \score(r({\bar{x}}))$.
In what follows, we assume that $f$ exists for the selected $\sigma$.

We defer all proofs to Appendix~\ref{appendix:theoretical}.

\subsection{Full Expressivity}
\label{sec:expressive}
The two results in this section state that \ourmodel{} is fully expressive and that some other models are not. 

\begin{theorem}[Expressivity]
\label{expressive}
For any ground truth over entities \entset{} and relations \relset{} containing $\lambda$ true tuples with $\maxar = \max_{r \in \relset}(|r|)$ as the maximum arity over all relations in \relset{}, there is a \ourmodel{} model with $n_w = \lambda$, $w = \maxar$, $d = \max(\maxar \lambda, \maxar)$, and $\sigma(x) = \frac{1}{1 + \exp(-x)}$ that accurately separates the true tuples from the false ones.
\end{theorem}

\begin{theorem}\label{th:m-TransH}
m-TransH, RAE, and NaLP are not fully expressive and have restrictions on the relations they can represent.
\end{theorem}

\subsection{Representing Relational Algebra with \ourmodel{}}\label{Operation-lemmas-section}

In this section, we describe the primitive relational algebra operations and prove how closely \ourmodel{} can represent each. 

\subsubsection{Renaming}
\op{Renaming} changes the name of one or more entities in a relation. A \op{renaming} operation can be written as the following logical rule, where $t$ is defined in terms of $s$.
\begin{equation}
\forall X_1,\dots, X_n  ~~~~~~~~
t({X_1,\dots,X_n}) \leftrightarrow s(X_{\pi(1)},\dots,X_{\pi(n)})
\end{equation}
For example,
$\forall~ X, Y, I ~~bought(X,Y,I) \leftrightarrow ~sold(Y,X,I)$
represents \op{renaming} relation (person X \emph{bought} item I from person Y) into relation (person Y \emph{sold} item I to person X).

\begin{theorem}[\op{Renaming}]
\label{lemma:renaming}
Given  permutation function $\pi$, and
relation~$s$, there exists a parametrization for relation~$t$ in \ourmodel{} such that
for  entities $x_1,\dots,x_n$, with arbitrary embeddings,
$$\score(t(x_1, \dots, x_n)) =  \score(s(x_{\pi(1)}, x_{\pi(2)},\dots,x_{\pi(n)}))$$
\end{theorem}

\subsubsection{Projection}
A \op{projection} operation that defines $t$ as a projection of $s$ can be written as the following (for $m<n$).
\begin{multline}
  \forall X_1,\dots ,X_m ~  
  t(X_1,\dots,X_m) ~ \\
  \leftrightarrow  \exists {X_{m+1},\dots ,X_n} ~~~~s(X_1,\dots,X_m,\dots,X_n)
\end{multline}
Note that \op{projection} can be paired with \op{renaming} to allow for arbitrary subsets and ordering of arguments. For example,
$$\forall X,I ~ seller(X,I) \leftrightarrow  \exists P ~bought(P,X,I)$$

\begin{theorem}[\op{Projection}]
\label{lemma:project}
For any relation $s$ on $n$ arguments there exists a parametrization for relation $t$ on $m<n$ arguments 
 in \ourmodel{} such that
for any arbitrary sequence $x_1,\dots, x_n$
$$\score(t(x_1,\dots, x_m)) \geq  \score(s(x_1,\dots, x_n))$$
\end{theorem}

Inequality is the best we can hope for, because multiple tuples with relation $s$ might project to the same tuple with relation $t$. 
The score of the tuple with relation $t$ should thus be greater than or equal to the maximum score for $s$. 

\subsubsection{Selection}
\op{Selection} returns the subset of tuples of a relation that satisfy a given condition.
Here, we consider equality conditions whereby a \op{selection} operation reduces the number of arguments, and has two forms defining $t$ as a selection of~$s$.
\begin{multline}
\forall X_1, \dots, X_n~~~~~~~
t(X_1,\dots ,X_{p-1},X_{p+1},\dots ,X_q,\dots ,X_{n})\\ \leftrightarrow \exists X_p ~~ s(X_1,\dots,X_n)  \land (X_{p} = X_{q})
\end{multline}
\begin{multline}
\forall X_1, \dots, X_n ~~~~~~~~~~~
t(X_1,\dots,X_{p-1},X_{p+1},\dots,X_{n}) \\ \leftrightarrow  \exists X_p ~~~s(X_1,\dots, X_n)  \land (X_{p} = c)  \textrm{~~~for fixed } c 
\end{multline}
%
%
For example,
\begin{multline*}
\forall X,Y~sold\_coffee(X,Y) \leftrightarrow \exists I ~sold(X,Y,I) \land (I=coffee)
\end{multline*}
in which $sold\_coffee(X,Y)$ is true if $X$ sold coffee to $Y$.

Observe that selecting tuples with the condition $X_p=X_q$ for arbitrary $p$ and $q$ is equivalent to first renaming the tuple so that $X_p$ is in position~$n$ and $X_q$ is in position $n-1$; and then performing a \op{selection} with the condition $X_{n-1}=X_n$ or $X_n=c$. 
Thus, to simplify our proofs, we show the \op{selection} operation for the case when $X_{n-1}=X_n$ and $X_n=c$. 

\begin{theorem}[\op{Selection 1}]
\label{lemma:sel1}
For arbitrary relation $s$, there exists a parametrization for relation $t$ in \ourmodel{} such that
for arbitrary entities $x_1,\dots,x_n$
\begin{equation*}
    \score(t(x_1,\dots,x_{n-1})) = \score(s(x_1,\dots,x_{n-1}, x_{n-1}))\\
\end{equation*}
\end{theorem}

\begin{theorem}[\op{Selection 2}]
\label{lemma:sel2}
For arbitrary relation $s$ and for a fixed constant $c$, there exists a parametrization for relation $t$ in \ourmodel{} such that
for arbitrary entities $x_1, \dots, x_n$
\begin{equation*}
    \score(t(x_1,\dots,x_{n-1})) = \score(s(x_1,\dots,x_{n-1}, c))
    \end{equation*}
\end{theorem}

\subsubsection{Set Union}
\op{Set union} operates on relations of the same arity, and returns a new relation containing the tuples
that appear in at least one of the relations. A \op{set union} operation can be written as the following logical rule with relation $t$ as the union of $s$ and $r$.
\begin{equation}
    \forall \bar{X} ~~~~~~~~ t(\bar{X}) \leftrightarrow s(\bar{X}) \lor r(\bar{X})
\end{equation}
For example,
\begin{multline*}
  \forall X_1, X_2,I~~~ traded(X_1, X_2,I) \\ \leftrightarrow ~sold(X_1,X_2,I) \lor bought(X_1,X_2,I) 
%
\end{multline*}
For a \ourmodel{} model to be able to represent the \op{set union} operation, first observe 
that any score for a tuple $t$ that represents the union of relations $r$ and $s$ 
depends on how dependent the two relations~$r$ and $s$ are. 
For example, if $s$ is a subset of $r$, then the score of $t$ is equal to that of $r$.
But since we do not know about such dependence relations in the data, then the best
we can hope for is a bound that shows that the score of~$t$ is \emph{at least} as high as the maximum 
score of either $r$ or $s$, as the following lemma states.

\begin{theorem}[\op{Set Union}]
\label{lemma:union}
For arbitrary relations $s$ and $r$ with the same arity, there exists a parametrization for relation $t$ in \ourmodel{} such that 
for arbitrary entity set $\bar{x}$
$$\score(t(\bar{x})) \geq \max(\score(s(\bar{x})), \score(r(\bar{x})))$$ 
\end{theorem}

\subsubsection{Set Difference}
\op{Set difference} operates on relations of the same arity, and
returns a new relation containing the tuples from the left relation that do not appear in the right one.
The \op{set difference} operation can be written as the following logical rule, where relation $t$ is set difference of $s$ and $r$.
\begin{equation}
    \forall \bar{X}  ~~~~~~~~ t(\bar{X}) \leftarrow s(\bar{X}) \land \neg r(\bar{X})
\end{equation}
For example, relation
$needs\_filter(X,Y)$ is true if $X$ bought coffee but did not buy coffee filter from $Y$.
\begin{multline*}
 \forall X,Y ~~needs\_filter(X,Y) \leftarrow bought\_coffee(X,Y) \land \\ 
 \neg bought\_filter(X,Y)
\end{multline*}

Similar to \op{set union}, the score of a \op{set difference} operator depends on how dependent the relations
$r$ and $s$ are. For the same reasons, the best we can hope for in this case is to show
that the score of $t$ is smaller than that of both $s$ and $\neg r$
(since $t(\bar{X})$ is true only when both $s(\bar{X})$ and $\neg r(\bar{X})$ are true, then the scores of the latter two must be higher).
In the lemma that follows, $f$ is the relation complement function described in the introduction of Section~\ref{sec:theoretical}.

\begin{theorem}[Set Difference]
\label{le:diff}
For arbitrary relations $r$ and $s$ with the same arity, if $f$ is a linear relation complement function and $f(\sigma(x)) = \sigma(c * x)$ with $c$ as a constant, 
there exists a parametrization for relation $t$ in \ourmodel{} such that
for arbitrary entities $x_1, \dots, x_n$ 
$$\score(t(\bar{x})) \leq \min(\score(s(\bar{x})), f(\score(r(\bar{x}))))$$ 
\end{theorem}

\begin{table*}
\footnotesize
\setlength{\tabcolsep}{2pt}
\caption{Knowledge hypergraph completion results on \acr{JF17K}, \acr{FB-auto} and \acr{m-FB15K} for baselines and the proposed method. Our method \ourmodel{} outperforms the baselines on all datasets. The 
sign 
``-'' indicates that the corresponding paper has not provided the results.
GETD is trained on a significantly smaller embedding dimension compared to the baselines to fit into the memory. Refer to Appendix~\ref{appndix:implementation} for implementation details of baselines and the proposed model.}
\label{table:FB-subsets}
\begin{center}
\setlength{\tabcolsep}{3pt}
\begin{tabular}{l|cccc|cccc|cccc}
\multicolumn{1}{c}{} & \multicolumn{4}{c}{\acr{JF17K}}
& \multicolumn{4}{c}{\acr{FB-auto}} & \multicolumn{4}{c}{\acr{m-FB15K}}\\
\cmidrule(lr){2-5} \cmidrule(lr){6-9} \cmidrule(lr){10-13}
Model    & MRR  & Hit@1 & Hit@3 & Hit@10 & MRR & Hit@1 & Hit@3 & Hit@10 & MRR & Hit@1 & Hit@3 & Hit@10 \\\hline
m-DistMult~\cite{HypE} & 0.463 & 0.372 & 0.510 & 0.634 & 0.784 & 0.745 & 0.815 & 0.845 & 0.705 & 0.633 & 0.740 & 0.844 \\
m-CP~\cite{HypE} &  0.392 & 0.303 & 0.441 & 0.560& 0.752& 0.704 & 0.785 & 0.837 & 0.680 & 0.605 & 0.715 & 0.828 \\
m-TransH~\cite{m-TransH} & 0.444 & 0.370 & 0.475 & 0.581 &  0.728 & 0.727 & 0.728 & 0.728& 0.623 & 0.531 & 0.669 & 0.809\\
RAE~\cite{RAE} & 0.310 & 0.219 & 0.334 &  0.504 & - & - & - & - & - & - & - & -\\
NaLP~\cite{NaLP} & 0.366 & 0.290 & 0.391 &  0.516 & - & - & - & - & - & - & - & -\\
GETD~\cite{GETD} & 0.151 & 0.104 & 0.151 & 0.258 & 0.367 & 0.254 & 0.422 & 0.601 & - & - & - & -\\
HSimplE~\cite{HypE} & 0.472 & 0.378 & 0.520 & 0.645 & 0.798 & 0.766 & 0.821 & 0.855 & 0.730 & 0.664 & 0.763 & 0.859\\
HypE~\cite{HypE} & 0.494 & 0.408 & 0.538 & 0.656 & 0.804 &
 0.774 &  0.823 & 0.856 & 0.777 & 0.725 & 0.800 & 0.881  \\
G-MPNN~\cite{yadati2020neural} & 0.501 & 0.425 & 0.537 & 0.660 & - & - & - & - & 0.779 & 0.732 & 0.805 & 0.894 \\
\hline
\ourmodel{}  (Ours) & \textbf{0.530} & \textbf{0.454} & \textbf{0.563} & \textbf{0.677} & \textbf{0.861} & \textbf{0.836} & \textbf{0.877} & \textbf{0.908} & \textbf{0.801} & \textbf{0.755} & \textbf{0.823} & \textbf{0.901}\\
\end{tabular}
\end{center}
\end{table*}


\setlength{\intextsep}{0pt}
\subsection{Representing Relational Algebra with HypE}\label{HypE-cannot-represent}
So far, we have established the theoretical properties of the proposed model and we showed that m-TransH, RAE, and NaLP (and G-MPNN in Section~\ref{sec:related-work}) are not fully expressive. 
In Appendix~\ref{appendix:relational-algebra-hype}, we further prove that HypE cannot represent all relational algebra operations. In particular, we show that HypE cannot represent \op{selection}.

\section{Experimental Setup}
In this section, we explain the datasets that are used in the experiments. We further explain the evaluation metrics used to compare models. We defer the implementation details of our model and baselines to Appendix~\ref{appndix:implementation}.
\subsection{Datasets}
To evaluate \ourmodel{}
for knowledge hypergraph completion, we conduct experiments on two classes of datasets.
 
\textbf{Real-world datasets.}  We use three real-world datasets for our experiments: 
\acr{JF17K} is proposed by \citet{m-TransH}, and \acr{FB-AUTO} and \acr{M-FB15K} are proposed by \citet{HypE}. 
Refer to Appendix~\ref{appendix:datasets} for statistics of the datasets.
 
\textbf{Synthetic dataset.} 
To study and evaluate the generalization power of models in a controlled environment, we also generate a synthetic dataset to use in our experiments. This practice has become common in recent years, with the creation of
several procedurally generated benchmarks to study the generalization power of models in different tasks. 
Examples of such datasets include CLEVR~\cite{johnson2017clevr} for images, TextWorld~\cite{cote2018textworld} for text data, and GraphLog~\cite{sinha2020evaluating} for graph data.
To create a benchmark for analyzing the relational algebraic generalization power of \ourmodel{} for hypergraph completion, we consider the following criteria. 

\begin{enumerate}[itemsep=0pt,topsep=0pt]
    \item \emph{Completeness}: The benchmarks must contain the desired relational algebra operations; in our case: \op{renaming}, \op{projection}, \op{selection}, \op{set union}, and \op{set difference}.
    \item \emph{Diversity}: The benchmark must contain a variety of relations that are the result of \emph{repeated application} of relational algebra operations having varying depth. 
    \item \emph{Compositional generalization}: The benchmark must contain relations that are the result of repeated application of operations of different types.
\end{enumerate}

To synthesize a dataset that satisfies the above conditions, we extend the Erd\H{o}s-R\'enyi model~\cite{erdos1959} for generating random graphs to directed edge-labeled hypergraphs. 
We use this hypergraph generation model to first generate a given number of true tuples; we then apply 
the five relational algebra operations to these tuples (repeatedly and recursively) to obtain new tuples with varying depth.
A detailed description of our  extension to the Erd\H{o}s-R\'enyi model and the full algorithm for generating
the synthetic dataset can be found in Appendix~\ref{appendix:datasets}.

\begin{table*}
\footnotesize
\caption{Breakdown performance of MRR across composite relations (based on a sequence of primitive operations) and of varying depth on the \synthdata{} dataset along with their statistics. 
}
\begin{center}
\setlength{\tabcolsep}{3pt}
\begin{tabular}{l|c|ccccc|cccc}

\multicolumn{2}{c}{}
& \multicolumn{5}{c}{Operation type}
& \multicolumn{4}{c}{Depth}\\
\cmidrule(lr){3-7}
\cmidrule(lr){8-11}
Model & All & elementary& \op{renaming} & \op{project} & \op{set union} & \op{set difference} & 1 & 2 & 3 & 4\\\hline
m-TransH~\cite{m-TransH} & 0.387 & 0.166 & 0.447 & 0.652 & 0.459 & 0.464 & 0.531 & 0.499 & 0.352 & 0.03\\
HypE~\cite{HypE} & 0.689& 0.335 & 0.877 & 0.856 & 0.888 & 0.894 & 0.881 & 0.872 & 0.897 & \textbf{0.976}\\
\hline
\ourmodel{} (Ours) & \textbf{0.709} & \textbf{0.336} & \textbf{0.882} & \textbf{0.877} & \textbf{0.932} & \textbf{0.923} & \textbf{0.887} & \textbf{0.950} & \textbf{0.945} & 0.938  \\\hline\hline
\#tuples & 5378 & 1833 & 458 & 762 & 1676 & 649 & 2075 & 1204 & 245 & 21 \\
\end{tabular}
\end{center}
\label{table:JF17KS}
\end{table*}
\subsection{Evaluation Metrics}
We evaluate the link prediction performance with two standard metrics Mean Reciprocal Rank (MRR) and Hit@k, $k \in \{1, 3, 10\}$.
Both MRR and Hit@k rely on the \emph{ranking} of a tuple $x \in \tau'_{\mathrm{test}}$ within a set of \emph{corrupted} tuples.

For each tuple  $r(x_1, \dots, x_n)$ in $\tau'_{\mathrm{test}}$ and each entity position~$i$ in the tuple,
we generate $|\entset|-1$ corrupted tuples by replacing 
the entity $x_i$ with each of the entities in $\entset \setminus \{x_i\}$.
For example, by corrupting entity $x_i$, we obtain a new tuple $r(x_1, \dots, x_i^c, \dots, x_n)$ where $x_i^c \in \entset \setminus \{x_i\}$.
Let the set of corrupted tuples, plus $r(x_1, \dots, x_n)$, be denoted by $\zeta_i(r(x_1,\dots,x_n))$.
Let $\rank_i(r(x_1, \dots, x_n))$ be the ranking of $r(x_1, \dots, x_n)$ within 
$\zeta_i(r(x_1, \dots, x_n))$ based on the score $\score(x)$ for each $x \in \zeta_i(r(x_1,\dots,x_n))$.
In an ideal knowledge hypergraph completion method, $\rank_i(r(x_1, \dots, x_n))$ is $1$ among all corrupted tuples
$\zeta_i(r(x_i,\dots,x_n))$.
We compute the MRR as
$
\frac{1}{N} \sum_{r(x_1, \dots,x_n) \in \tau'_{\mathrm{test}}} 
\sum_{i=1}^{n}\frac{1}{\rank_{i}(r(x_1, \dots,x_n))}
$
where $N = \sum_{r(x_1,\dots x_n) \in \tau'_{\mathrm{test}}} |r|$ is the number of prediction tasks. 
Hit@k measures the proportion of tuples in $\tau'_{\mathrm{test}}$ that rank among the top $k$ in their corresponding corrupted sets. 
Following \citet{TransE}, we remove all corrupted tuples that are in $\tau'$ from our computation of MRR and Hit@k. 

%
%
\section{Experiments}\label{experiments}
We organize our experiments into three groups satisfying three different objectives.
The goal of the first set of experiments is to evaluate the proposed method on real datasets and compare its performance to that of existing work. 
Our second goal is to test the ability of \ourmodel{} to represent the primitive
relational algebra operations. For this purpose, we evaluate our model on a synthetic dataset in which tuples are generated
by repeated application of relational algebra operations. This provides us with a controlled environment where
we know the ground-truth about what operation(s) each relation represents.
The final set of experiments is an ablation study that examines the effect of window size.

\subsection{Results on Real Datasets}
\label{sec:real_experiments}
We evaluate \ourmodel{} on three public datasets \acr{JF17K}, \acr{FB-AUTO}, and \acr{M-FB15K} and compare its performance to that of existing models. 
Our experiments show that \ourmodel{}  outperforms existing knowledge hypergraph completion models 
across all three datasets \acr{JF17K}, \acr{FB-auto}, and \mbox{\acr{m-FB15K}}. Results are summarized in Table~
\ref{table:FB-subsets}. 

\subsection{Results on Synthetic Datasets}
\label{sec:synthetic_experiments}
To evaluate our model in a controlled environment, we create a synthetic dataset whose statistics (e.g., number of entities, number of relations, number of tuples per relation) is proportional to that of \acr{JF17K}. We call this synthetic dataset {\emph{\synthdata{}}}.
We break down the performance of our model based on the type of relational algebra operation and its depth. 
As a first step of the synthetic dataset generation, we create a list of tuples at random. The relations involved in these tuples
are called \emph{elementary} relations, as they are not generated based on any relational algebra operation. 
We then create a set of tuples that are based on \emph{composite} relations: those defined as a sequence of primitive relational algebra operations that are applied recursively to elementary relations. 
We let the \emph{type} of a composite relation be the last operation applied.
We collect our results by type and list them under the corresponding column in Table~\ref{table:JF17KS}.
Note that choice of the \emph{last} operation for defining the type does not change the overall performance of the model; it merely helps us break down the performance to gain insight. As there is no obvious way of defining the type of a composite relation, we experimented with letting it be determined by the first, last, or all operations; our experiments showed little variation between the results.
We define the \emph{depth} of a relation recursively. An elementary relation has depth of $0$. A relation that is the result of a unary operation (e.g., \op{projection} or \op{renaming}) has a depth of one plus the depth of the input relation. For a relation that is the result of a binary operation (e.g., \op{set union} or \op{set difference}), the depth is one plus the maximum depth of the input relations.

The results of our experiments on \synthdata{} are summarized in Table~\ref{table:JF17KS} and 
show that our proposed model outperforms the state-of-the-art in almost all cases. 
As the decomposed performance shows, the improvement to the general result 
is due mostly to improvements in the performance of \op{renaming}, \op{projection}, \op{selection}, \op{set union}, and \op{set difference} as well as improvements for elementary relations. 
These results confirm that our theoretical findings are in line with the practical results.
Note that the reason why we do not have tuples for the \op{selection} operation in the test set is that, by nature, \op{selection} generates very few tuples; and as the split of train/test/validation in \synthdata{} is done at random, the chance of having \op{selection} tuples in the test set is very low and did not occur when we synthesized the dataset (\op{selection} tuples are still present in the train data).

\subsection{Ablation Study on Varying Window Sizes}
\label{sec:ablation_experiments}

We compare the performance of \ourmodel{} when trained with different window sizes to 
empirically study its effect on the results.
The outcome of the study is summarized in Figure~\ref{chart:ws}, which shows the MRR of \ourmodel{}
with the best set of hyperparameters for window sizes of the first five divisors of the embedding dimension. In our framework, a larger window size implies more interaction between the elements of the embeddings.
As the chart shows, the performance of \ourmodel{} 
is at a maximum between window sizes of 2 and 4, but is lower at smaller and larger sizes.
A window size of~1 is the same as having no windows. This result confirms the importance of having windows to ensure interaction between the embedding elements, and also suggests that excessive interaction of embedding elements cannot further improve performance. Therefore, window size is a sensitive hyperparameter that needs to be tuned for the best performance given the dataset.

\vspace{0.5cm}

\begin{figure}[h]
\begin{center}
    \includegraphics[width=0.38\textwidth]{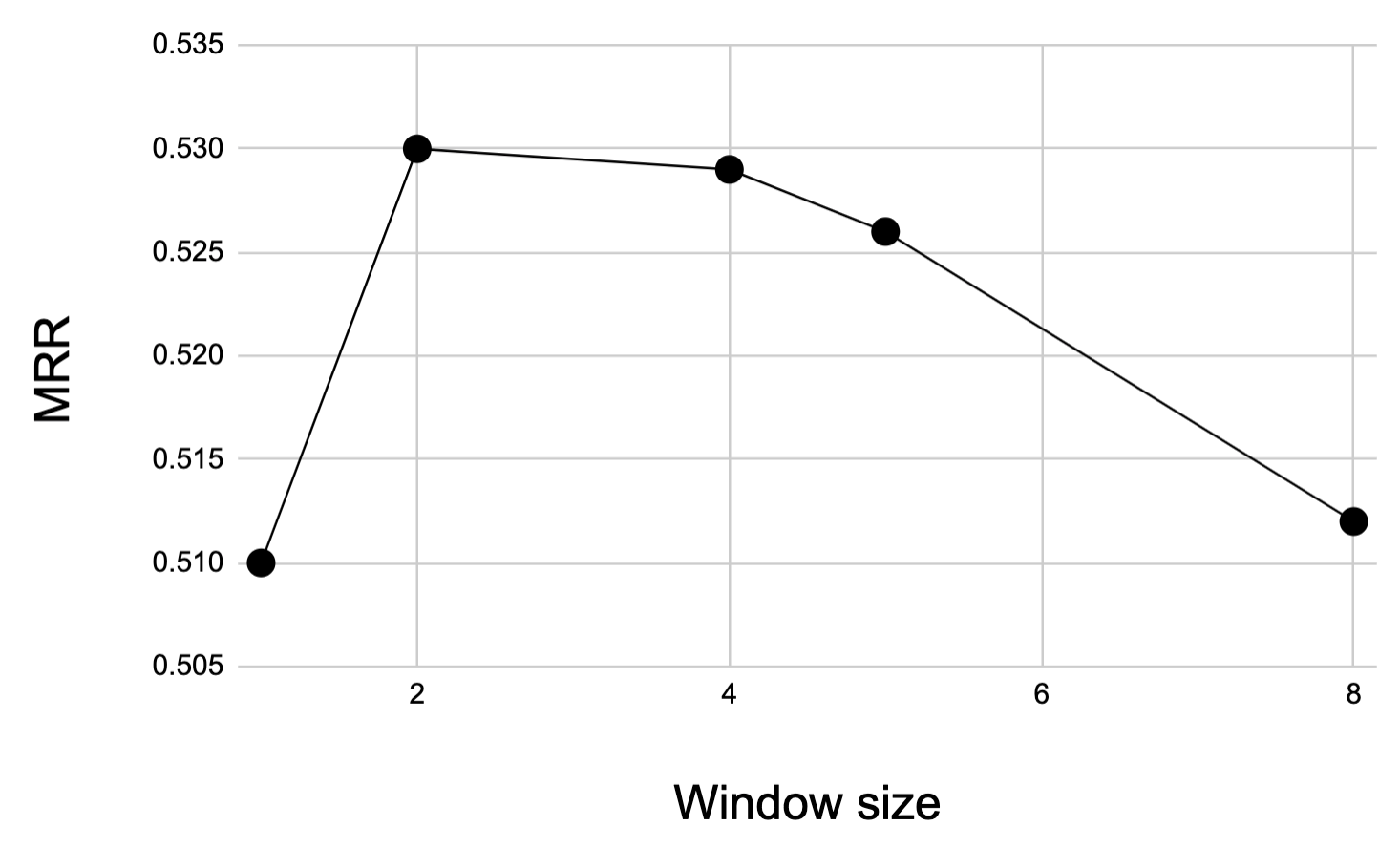}
  \end{center}
  \caption{MRR of \ourmodel{} for window sizes as the first five divisors of the embedding size (200) on \acr{JF17K}.}
  \label{chart:ws}
\end{figure}

\section{Conclusion}
In this work, we introduce \ourmodel{} for reasoning in knowledge hypergraphs.
To design a powerful method, we build on the primitives of relational algebra, which is the calculus of relational models. We prove that \ourmodel{} can represent the primitive relational algebra operations \op{renaming}, \op{projection}, \op{selection}, \op{set union}, and \op{set difference}, and is fully expressive. The results of our experiments on real and synthetic datasets are consistent with the theoretical findings.

\bibliography{bibliography}
\bibliographystyle{icml2021}

\appendix

\setcounter{theorem}{0}

\section{Theoretical Analysis}\label{appendix:theoretical}
The current section groups the theoretical analysis of our work into four parts. 
In particular,
Section~\ref{sec:full-express-reale} proves that \ourmodel{} is fully expressive. Section~\ref{sec:full-express-baselines} proves that m-TransH, RAE, and NaLP are not fully expressive (G-MPNN scoring function and its non-expressiveness were discussed in Section 2 of the main paper).
Section~\ref{sec:relational-algebra-reale} shows how closely \ourmodel{} can represent relational algebra operations.
Section~\ref{appendix:relational-algebra-hype} further proves that HypE cannot represent all relational algebra operations (in particular, we show that HypE cannot represent \op{selection}).

For completeness, we restate the theorems and the scoring function.
In what follows, we use lower case 
$x_1, \dots, x_n$ to denote particular entities, $\bar{x}$ a sequence of particular entities, and $\vc{x}$ the embedding of $x$.
Recall that our model embeds each entity $x_i$ into a vector $\vc{x_i} \in [0,1]^d$ of length~$d$, and each relation $r$ into a matrix $\vc{r} \in \realset^{|r| \times d}$.
Recall that $w$ denotes the window size, $n_w = \lfloor \frac{d}{w} \rfloor$  the number of windows, and $b_r^{j}$ the bias term of relation $r$ for the $j^{th}$ window, for all $j=0,\dots,n_w-1$.
Equation~\ref{eq:main} defines \ourmodel{}'s score of a tuple $r(x_1, x_2, \dots, x_n)$, where~$\sigma$ is a monotonically-increasing nonlinear function that is differentiable almost everywhere and $\theta$ is the set of all entity and relation embeddings.

\begin{equation}\label{eq:main}
\begin{split}
&\score(r(x_1, \dots, x_n)) = \\ 
&\frac{1}{n_w}\sum_{j=0}^{n_w-1} \sigma\left(b_r^{j} + \sum_{i=1}^{|r|} \sum_{k=0}^{w-1} \vc{x_i}[jw + k] \times \vc{r}[i][jw + k]\right)
\end{split}
\end{equation}

\subsection{Full Expressivity of \ourmodel{}}\label{sec:full-express-reale}
The following result proves that there exists a setting of the parameters for which \ourmodel{} can separate true and false tuples for arbitrary input. In particular, we show it for the case where $\sigma$ is the sigmoid function.

\begin{theorem}[Expressivity]
\label{expressive}
For any ground truth over entities \entset{} and relations \relset{} containing $\lambda$ true tuples with $\maxar = \max_{r \in \relset}(|r|)$ as the maximum arity over all relations in \relset{}, there is a \ourmodel{} model with $n_w = \lambda$, $w = \maxar$, $d = \max(\maxar \lambda, \maxar)$, and $\sigma(x) = \frac{1}{1 + \exp(-x)}$ that accurately separates the true tuples from the false ones.
\end{theorem}

\begin{figure*}[t]
\label{fig:expressive}
\begin{center}
    \includegraphics[width=0.8\textwidth, trim={0 0.0cm 0 0.0cm},clip]{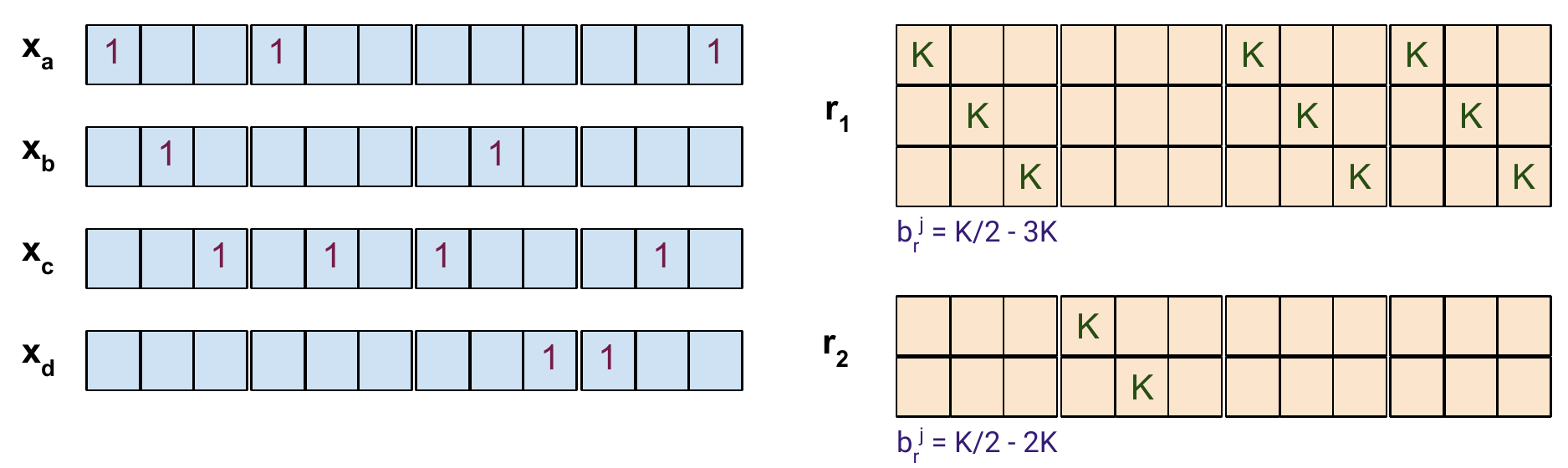} 
  \end{center}
  \caption{An example of a \ourmodel{} embedding assignments for true tuples $\tau_0 = r_1(x_a, x_b, x_c)$, $\tau_1 = r_2(x_a, x_c)$, $\tau_2 = r_1(x_c, x_b, x_d)$ and $\tau_3 = r_1(x_d, x_c, x_a)$. Here, the number of true tuples $\lambda=n_w=4$, maximum arity $\omega = 3$. Cells that are set to zero are left empty for better readability.
As an example, given that $x_b$ is in position $1$ of $\tau_2$, 
we set cell  $2\times \omega + 1 = 7$ of $x_b$ to $1$. We further set the value of $r_1$ at positions [$1$][$7$] to $K$.}
  \label{chart:ws}
\end{figure*}

\begin{proof}
Let $T=\{\tau_0,\tau_1,\dots,\tau_{\lambda-1}\}$ be all the true tuples defined over \entset{} and \relset{}. 
To prove the theorem, we show an assignment of embedding values for each of the entities and relations
in $T$ such that the scoring function of \ourmodel{} is as follows.
\[
  \phi(\tau) 
  \begin{cases}
                                  \geq \frac{1 - \epsilon}{\lambda} & \text{if $\tau \in T$} \\
                                  < \epsilon  & \text{otherwise}
  \end{cases}
\]

Here, $\epsilon$ is an arbitrary small value such that $\epsilon < \frac{1}{2+\lambda^2}$. 
Observe that $\lambda$ is a positive integer.
Therefore, $\epsilon$ and $\frac{1 - \epsilon}{\lambda}$ never meet.

We first consider the case where $\lambda > 0$. 

Let $K = 2 \times \sigma^{-1}(1 - \epsilon)$. Then, $\sigma(\frac{K}{2}) = 1 - \epsilon$ and $\sigma(\frac{-K}{2}) = \epsilon$.
We begin the proof by first describing an assignment of the embeddings of each of the entities and relations in \ourmodel{};
we then proceed to show that with such an embedding, \ourmodel{} accurately separates the true tuples from the false ones.

We consider the embeddings of entities to be $n_w=\lambda$ blocks of size $w$ each, such that each block~$i$ is conceptually associated with true tuple $\tau_i$ for all $0 \leq i < \lambda$. 
Then, for a given entity $x_m$ at position $m$ of tuple $\tau_i$, 
we set the value $\vc{x_m}[$iw + m$]$ to $1$, for all $0 \leq i < \lambda$. 
All other values in $\vc{x_m}$ are set to zero.
For the relation embeddings, first recall that these embeddings are matrices of dimension $|r| \times w\lambda$, 
and we consider each of the $|r|$ rows of this matrix to be $\lambda$ blocks of size $w$, where $|r| > 1$ is the arity of the given relation.
If a given relation $r$ appears in true tuple $\tau_i$, we set the $|r|$ values at position [$iw + k$][$iw + k$] to $K$, for all $0 \leq k < |r|$ and $0 \leq i < \lambda$; all other values in the relation embedding are set to zero.
Finally, we set all the bias terms in our model to $b_r^{j} = \frac{K}{2} - |r|\times K$, for all $0 \leq j < \lambda$.
As an example, consider Figure~\ref{fig:expressive}, in which
the first true tuple $\tau_0$ is defined on relation $r_1$; thus the embedding of $r_1$ has all ones in its first $3 \times 3$ block (note that since $r_1$ is the relation in the third and fourth true tuples, then the third and fourth blocks of the embedding of $r_1$ are also filled with $1$s). $\tau_0$ has entity $x_b$ at its second position; hence the embedding of $x_b$ has a $1$ at its second position.
We claim that with such an assignment, the score of tuples that are true is  $\geq \frac{1 - \epsilon}{\lambda}$ and $< \epsilon$ otherwise.

To see why this assignment works, first observe that our scoring function $\score{}$ is an averaging of 
the sum of $\lambda$ sigmoids; 
each sigmoid is defined on an embedding block where we sum the bias term with the sum of pairwise product between the entity and relation embeddings of the given block.

Let $\tau_p = r(x_1,\dots,x_m)$ be a true tuple and observe the embeddings of its entities and relations. 
The blocks at position $p$ of the embeddings of each of $x_1, \dots, x_m$  
contain exactly one value $1$ each (with the rest being zero); and the
block at position $p$ of the relation embedding for $r$ contains all $K$s. 
With such an assignment, the block at position $p$ will contribute the following sigmoid to the scoring function.
\begin{equation}
\label{eq:express_true}    
\sigma\left(b_r^p + |r|K\right) = \sigma\left(\frac{K}{2} - (|r|K) + (|r| K)\right) = \sigma\left(\frac{K}{2}\right) = 1 - \epsilon
\end{equation}

All other blocks $q \neq p$ in $\tau_p$ contain at least one entity embedding block (for each of $x_a,\dots, x_m$) to be all zeros (because otherwise they will be duplicate tuples in $T$).
Assume that there are $c>0$ entity blocks that are all zeros for block $q$.
Then, we have $|r| - c$ blocks that contain exactly one cell with value $1$, for all $0 \leq q < \lambda$, $q \neq p$.
Looking at $\tau_0$ in the example of Figure~\ref{fig:expressive}, if $q=1$ (second block), then $c=1$ since the second block of $x_b$ is all zeros.

Thus, any block $q$ in $\tau_p$ will contribute to the scoring function in one of two ways,
depending on whether or not $r$ is a relation in $\tau_q$.

If $r$ happens to be a relation in $\tau_q$ (e.g. $r_1$ is a relation in $\tau_0$, $\tau_2$ and $\tau_3$ in Figure~\ref{fig:expressive}), then the $|r| \times w$-sized block at position $q$ of the relation embedding is all set to $K$; 
and since $c > 1$,
\begin{equation}
\begin{split}
\label{eq:express_false_1}    
& \sigma\left(b_r^q +  (|r|-c)K\right) = \\
&\sigma\left(\frac{K}{2} - (|r| K) + (|r|-c) K \right) = \sigma\left(\frac{K}{2}-cK\right) <  \epsilon
\end{split}
\end{equation}

If $r$ is not a relation in $\tau_q$, then the block at position $q$ of the relation embedding is all zeros. 
In this case,
\begin{equation}
\label{eq:express_false_2}    
\sigma\left(b_r^q + 0\right) = \sigma\left(\frac{K}{2} - (|r| K)\right) <  \epsilon
\end{equation}

In the end, the score of any true tuple $\score{}(\tau_p)$ will be the sum of $\lambda$ sigmoids divided by $\lambda$
such that exactly one of these sigmoids (block at position $p$) has a value $1-\epsilon$, while the other $\lambda - 1$ sigmoids  $< \epsilon$.
The score of a true tuple $\tau_p$ can thus be bounded as follows.
\begin{equation}
\label{eq:express_true}    
\score{(\tau_p)} >= \frac{1-\epsilon}{\lambda}
\end{equation}

In the case of a false tuple, all the sigmoids will yield values $< \epsilon$ (as they will be of the forms in either of Equations~\ref{eq:express_false_1}, or \ref{eq:express_false_2}).
The score of a false tuple $\tau_f$ can thus be bounded as follows.
\begin{equation}
\label{eq:express_false}    
\score{(\tau_f)} < \frac{\lambda\epsilon}{\lambda} =  \epsilon
\end{equation}

To complete the proof, we consider the case when $\lambda = 0$.
In this case, we let the entity and relation embeddings be blocks of arbitrary size and set all values to zero.  
We set the bias terms as before.
Then, the score of any tuple will be $\score{(\tau_f)} < \epsilon$ (Equation~\ref{eq:express_false}), which is what we want.
This completes the proof.
\end{proof}

\subsection{Full Expressivity of Some Baselines}\label{sec:full-express-baselines}
The following results prove that some of the baselines are not fully expressive and have severe restrictions on the types of relations they can model. \citet{GETD} discuss some of these limitations, and here we address them more formally.

\begin{theorem}\label{thm:m-TransH}
m-TransH, RAE, and NaLP are not fully expressive and have restrictions on the relations these approaches can represent.
\end{theorem}

The proof of the above theorem follows from Lemma~\ref{lemma:m-TransH} and Lemma~\ref{nalp} below.

\begin{customlemma}{2.1}
\label{lemma:m-TransH}
m-TransH and RAE are not fully expressive and have restriction on what relations these
approaches can represent.
\end{customlemma}

\begin{proof}
\citeauthor{kazemi2018simple} (\citeyear{kazemi2018simple}) prove that TransH~\cite{TransH} is not fully expressive and explore its restrictions. As m-TransH reduces to TransH for binary relations, it inherits all its restrictions and is not fully expressive. RAE also follows the same strategy as m-TransH in modeling the relations. Therefore, both m-TransH and RAE are not fully expressive.
\end{proof}

\begin{customlemma}{2.2}
\label{nalp}
NaLP is not fully expressive.
\end{customlemma}

\begin{proof}
NaLP first concatenates the embeddings of entities and the embedding of their corresponding roles in the tuples, then applies to them the following functions: 1D convolution, projection layer, minimum, and another projection layer. Looking carefully at the output of the model, the NaLP scoring function for a tuple $r(x_1, \dots, x_n)$ is in the form of $|\vc{P_1} \vc{x_1} + \dots + \vc{P_n} \vc{x_n} + \vc{r}|_1$ with $\vc{P_i}$ as learnable diagonal matrices with some shared parameters. In the binary setup ($n = 2$), the score function of NaLP is $|\vc{P_1} \vc{x_1} + \vc{P_2} \vc{x_2} + \vc{r}|_1$. \citeauthor{kazemi2018simple} (\citeyear{kazemi2018simple}), however, proved that translational methods having a score function of $|\vc{P_1} \vc{x_1} - \alpha \vc{P_2} \vc{x_2} + \vc{r}|_i$ are not fully expressive and have severe restrictions on what relations these approaches can represent. NaLP has the same score function with $\alpha = -1$ and $i = 1$ and therefore is not fully expressive.
\end{proof}

\subsection{Representing Relational Algebra with \ourmodel{}}\label{sec:relational-algebra-reale}
 
\begin{theorem}[\op{Renaming}]
\label{le:renaming}
Given  permutation function $\pi$, and
relation $s$, there exists a parametrization for relation $t$ in \ourmodel{} such that
for  entities $x_1,\dots,x_n$, with arbitrary embeddings
$$\score(t(x_1, \dots, x_n)) =  \score(s(x_{\pi(1)}, x_{\pi(2)},\dots,x_{\pi(n)}))$$
\end{theorem}

\begin{proof}
To prove the above statement, we show that given entity embeddings $\vc{x_1},\dots,\vc{x_n}$ and an embedding for $s$, there exists a parametrization for $t$ that satisfies the above equality. 

We claim that the following settings for $t$ are enough to show the theorem.
\begin{enumerate}
    \item $\vc{t}[\pi(i)][k] = \vc{s}[i][k] ~~~ \forall ~ 1 \leq i \leq n, ~\forall~ 0 \leq k < d$, and
    \item $b^j_{t} = b^j_s ~~~ \forall ~ 0 \leq j < n_w$ 
\end{enumerate}%

To see why, we simply expand the score function of $s$ and replace the values for $s$ by that of $t$ as described, to obtain the score for $t$.

\begin{dmath*}
    \score(s(x_{\pi(1)}, x_{\pi(2)},\dots,x_{\pi(n)})) =
    \frac{1}{n_w}\sum_{j=0}^{n_w-1} \sigma(b_s^{j} + \sum_{i=1}^{|s|} \sum_{k=0}^{w-1} \vc{x}_{\pi(i)}[j\times w + k] \times \vc{s}[i][j\times w + k]) = 
    \frac{1}{n_w}\sum_{j=0}^{n_w-1} \sigma(b_t^{j} + \sum_{i=1}^{|s|} \sum_{k=0}^{w-1} \vc{x}_{\pi(i)}[j\times w + k] \times \vc{t}[\pi(i)][j\times w + k]) = 
    \frac{1}{n_w}\sum_{j=0}^{n_w-1} \sigma(b_t^{j} + \sum_{i=1}^{|t|} \sum_{k=0}^{w-1} \vc{x}_{i}[j\times w + k] \times \vc{t}[i][j\times w + k]) =
    \score(t(x_1, x_2,\dots,x_n)) %
\end{dmath*}%
\end{proof}

\begin{theorem}[\op{Projection}]
For any relation $s$ on $n$ arguments there exists a parametrization for relation $t$ on $m<n$ arguments 
 in \ourmodel{} such that
for any arbitrary sequence $x_1,\dots, x_n$
$$\score(t(x_1,\dots, x_m)) \geq  \score(s(x_1,\dots, x_n))$$
\end{theorem}

\begin{proof}
To prove the above statement, we first expand the score function of each side of the inequality. 

\begin{dmath}
    \score(t(x_1,\dots, x_m)) =
    \frac{1}{n_w}\sum_{j=0}^{n_w-1} \sigma(b_t^{j} + \sum_{i=1}^{m} \sum_{k=0}^{w-1} \vc{x}_{i}[j\times w + k] \times \vc{t}[i][j\times w + k]) 
\end{dmath}

\begin{dmath}
    \score(s(x_1,\dots, x_n)) =
    {\frac{1}{n_w}\sum_{j=0}^{n_w-1} \sigma(b_s^{j} + 
    \sum_{i=1}^{m} \sum_{k=0}^{w-1} \vc{x}_{i}[j\times w + k] \times \vc{s}[i][j\times w + k]} + ~ 
    \sum_{i=m+1}^{n} \sum_{k=0}^{w-1} \vc{x}_{i}[j\times w + k] \times \vc{s}[i][j\times w + k]) 
\end{dmath}

The theorem holds with the following assignments.

\begin{enumerate}
    \item $\vc{t}[i][k] = \vc{s}[i][k] ~~~ \forall ~ 1 \leq i \leq m, ~\forall ~ 0 \leq k < d$, and
    \item $b^j_{t} = b^j_s + \displaystyle\sum_{i=m+1}^{n} \sum_{k=0}^{w-1} \max(0, \vc{s}[i][k])  ~~~
          \forall ~ 0 \leq j < n_w $
\end{enumerate}%
\end{proof}

\begin{theorem}[\op{Selection 1}]
\label{lemma:selection1}
For arbitrary relation $s$, there exists a parametrization for relation $t$ in \ourmodel{} such that
for arbitrary entities $x_1,\dots,x_{n-1}$
\begin{equation*}
    \score(t(x_1,\dots,x_{n-1})) = \score(s(x_1,\dots,x_{n-1}, x_{n-1}))\\
\end{equation*}
\end{theorem}

\begin{proof}
To prove the above statement, we first expand the score function of $\score$ for the right side of the equality. 
\begin{equation}
\label{sel1:score s}
\begin{split}
    & \score(s(x_1,\dots,x_{n-1}, x_{n-1})) \\ =
    & \frac{1}{n_w}\sum_{j=0}^{n_w-1} \sigma(b_s^{j} + 
    \sum_{i=1}^{n-2} \sum_{k=0}^{w-1} \vc{x}_{i}[j\times w + k] \times \vc{s}[i][j\times w + k] \\
    & +
    \sum_{k=0}^{w-1} \vc{x}_{n-1}[j\times w + k] \times \vc{s}[n-1][j\times w + k] \\
    & +
    \sum_{k=0}^{w-1} \vc{x}_{n-1}[j\times w + k] \times \vc{s}[n][j\times w + k]) \\
    = &
    ~(\text{grouping terms}) \\
    & \frac{1}{n_w}\sum_{j=0}^{n_w-1} \sigma(b_s^{j} + 
    \sum_{i=1}^{n-2} \sum_{k=0}^{w-1} \vc{x}_{i}[j\times w + k] \times \vc{s}[i][j\times w + k] \\
    & +
    {\sum_{k=0}^{w-1} \vc{x}_{n-1}[j\times w + k] \times (\vc{s}[n - 1][j\times w + k]} + \vc{s}[n][j\times w + k]))
\end{split}
\end{equation}

For the lemma to hold, the above score has to be equal to that of $t$, which is described as follws.
\begin{dmath}
\label{sel1:score r}
    \score(t(x_1,\dots,x_{n-1})) = 
    {\frac{1}{n_w}\sum_{j=0}^{n_w-1} \sigma(b_t^{j} + 
    \sum_{i=1}^{n-2} \sum_{k=0}^{w-1} \vc{x}_{i}[j\times w + k] \times \vc{t}[i][j\times w + k]} +
    {\sum_{k=0}^{w-1} \vc{x}_{n-1}[j\times w + k] \times \vc{t}[n-1][j\times w + k])}
\end{dmath}

The scores in Equations~\ref{sel1:score s} and \ref{sel1:score r} are equal when the embedding and bias of $t$ are set as follows.

\begin{enumerate}
    \item $\vc{t}[i][k] = \vc{s}[i][k] ~~~ \forall ~ 1 \leq i \leq n - 2, ~\forall~ 0 \leq k < d$, and 
    \item $\vc{t}[n-1][k] = \vc{s}[n-1][k] + \vc{s}[n][k]  ~~~   ~\forall~0 \leq k < d$, and 
    \item $b^j_{t} = b^j_s  ~~~   \forall~0 \leq j < n_w $%
\end{enumerate}
\end{proof}

\begin{theorem}[\op{Selection 2}]
For arbitrary relation $s$ and for a fixed constant $c$, there exists a parametrization for relation~$t$ in \ourmodel{} such that
for arbitrary entities $x_1, \dots, x_{n-1}$
\begin{equation*}
    \score(t(x_1,\dots,x_{n-1})) = \score(s(x_1,\dots,x_{n-1}, c))
    \end{equation*}
\end{theorem}

\begin{proof}
Using similar score expansions as in the proof of Lemma~\ref{lemma:selection1}, we can rewrite the scores of the two relations as follows.
\begin{equation}\label{sel:score s}
\begin{split}
    & \score(s(x_1,\dots,x_{n-1}, c)) \\ & =
     \frac{1}{n_w}\sum_{j=0}^{n_w-1} \sigma(b_s^{j} + \sum_{i=1}^{n-2} \sum_{k=0}^{w-1} \vc{x}_{i}[j\times w + k] \times \vc{s}[i][j\times w + k] \\
    & +
    \sum_{k=0}^{w-1} \vc{x}_{n-1}[j\times w + k] \times \vc{s}[n-1][j\times w + k] \\
    & + 
    \sum_{k=0}^{w-1} \vc{c}[j\times w + k] \times \vc{s}[n][j\times w + k])
\end{split}
\end{equation}
\begin{dmath}\label{sel:score r}
    {\score(t(x_1,\dots,x_{n-1}))} = 
    \frac{1}{n_w}\sum_{j=0}^{n_w-1} \sigma(b_t^{j} + 
    {\sum_{i=1}^{n-2} \sum_{k=0}^{w-1} \vc{x}_{i}[j\times w + k] \times \vc{t}[i][j\times w + k]} +
    {\sum_{k=0}^{w-1} \vc{x}_{n-1}[j\times w + k] \times \vc{t}[n-1][j\times w + k])} 
\end{dmath}
The scores in Equations~\ref{sel:score s} and \ref{sel:score r} are equal when the embedding and bias of $t$ are as set follows.
%
\begin{enumerate}
    \item $\vc{t}[i][k] = \vc{s}[i][k] ~~~ \forall ~ 1 \leq i \leq n-1, ~\forall~ 0 \leq k < d$, and 
    \item $b^j_{t} = b^j_s + \sum_{k=0}^{w-1} \vc{s}[n][j\times w + k] \times \vc{c}[j\times w + k] ~~~   \forall~0 \leq j < n_w $
\end{enumerate}
\end{proof}

\begin{theorem}[\op{Set Union}]
For arbitrary relations $s$ and $r$ with the same arity, there exists a parametrization for relation $t$ in \ourmodel{} such that 
for arbitrary entity set $\bar{x}$
$$\score(t(\bar{x})) \geq \max(\score(s(\bar{x})), \score(r(\bar{x})))$$ 
\end{theorem}

\begin{proof}
Given that \ourmodel{} embeds entities in non-negative vectors and examining the scoring functions of each of the relations $t$, $r$, $s$, it can be observed that the above
inequality holds by setting the following values.
\begin{enumerate}
\item $t[i][k] = \max(s[i][k], r[i][k])$  ~~~ $\forall ~~ 1\leq i < |\bar{x}|$ and $0 \leq k < d$.
\item $b^j_t = \max(b^j_s, b^j_r)$ ~~~ $\forall ~~ 0 \leq j < n_w$.
\end{enumerate}\end{proof}

In the lemma that follows, recall that the \emph{relation complement} function (if it exists) is some linear function $f(\score(r({\bar{x}}))) = \score(\neg r({\bar{x}}))$ for arbitrary relation $r$ and entities $\bar{x}$ that also has the form $f(\sigma(x)) = \sigma(c \times x)$ with $c$ as a constant. For instance, when $\sigma$ is sigmoid, $f(x) = 1 - x$ and $c = -1$ ($f(\sigma(x)) = 1 - \sigma(x) = \sigma(-x)$) and when $\sigma$ is tanh, $f(x) = -x$ and $c = -1$ ($f(\sigma(x)) = -\sigma(x) = \sigma(-x)$).

\begin{theorem}[\op{Set Difference}]
\label{le:diff}
For arbitrary relations $r$ and $s$ with the same arity, if $f$ is a linear relation complement function and $f(\sigma(x)) = \sigma(c \times x)$ with $c$ as a constant, 
there exists a parametrization for relation $t$ in \ourmodel{} such that
for arbitrary entities $x_1, \dots, x_n$ 
$$\score(t(\bar{x})) \leq \min(\score(s(\bar{x})), f(\score(r(\bar{x}))))$$ 
\end{theorem}

\begin{proof}

As $f$ is the relation complement, we have:
\begin{dmath*}
\score(\neg r(x_1, \dots, x_n)) = f(\score(r(x_1, \dots, x_n)))
\end{dmath*}
As $f$ is linear, we can distribute it inside the summation as follows:
\begin{dmath*}
f(\score(r(x_1, \dots, x_n))) = 
\frac{1}{n_w}\sum_{j=0}^{n_w-1} f(\sigma(b_r^{j} + \sum_{i=1}^{|r|} \sum_{k=0}^{w-1} \vc{x_i}[jw + k] \times \vc{r}[i][jw + k]))
\end{dmath*}

Now, as $f(\sigma (x)) = \sigma(c \times x)$, we can distribute it inside the $\sigma$ as follows:

\begin{dmath*}
f(\score(r(x_1, \dots, x_n))) = 
\frac{1}{n_w}\sum_{j=0}^{n_w-1} \sigma \left(b_r^{j} \times c + \sum_{i=1}^{|r|} \sum_{k=0}^{w-1} \vc{x_i}[jw + k] \times \vc{r}[i][jw + k]\right \times c)
\end{dmath*}

Therefore, for the above inequality to hold, the bias terms and embedding values of $t$ must be at most that of each of~$s$ and the complement of the score of~$r$.
Examining the scoring functions of each of the relations $t$, $r$, $s$, it can be observed that the lemma holds when the following is set. 
\begin{enumerate}
\item $t[i][k] = \min(s[i][k], r[i][k] \times c)$  ~~~ $\forall ~~ 1\leq i < |\bar{x}|$ and $0 \leq k < d$.
\item $b^j_t = \min(b^j_s, b^j_r \times c)$ ~~~ $\forall ~~ 0 \leq j < n_w$.
\end{enumerate}
\end{proof}

\subsection{Representing Relational Algebra with HypE}\label{appendix:relational-algebra-hype}
So far, we showed that \ourmodel{} is fully expressive and can represent
the relational algebra operations \op{renaming}, \op{projection}, 
\op{selection}, \op{set union}, and \op{set difference}.
We also showed that most other models for knowledge hypergraph completion
are not fully expressive.
In this section, we show that even as HypE is fully expressive in general,
it cannot represent some relational algebra operations (namely, \op{selection}) while at the same time
retaining its full expressivity. In special cases, such as when all tuples are false, it obviously can, however, we show that in the general case it cannot. 

We proceed by first showing in Theorem~\ref{theorem:Hype selection1} that HypE cannot represent \op{selection} for \emph{arbitrary} entity and relation embeddings while retaining full expressivity. Now, one might think that even if representing \op{selection} is not possible for all embeddings, there might be some embedding space for which HypE would be able to represent \op{selection}. In Theorem~\ref{theorem:Hype selection2} we show that when the embedding size is less than the number of entities (as it usually is the case), then there is no setting for which HypE would be able to represent \op{selection}.

Recall that HypE embeds an entity $x_i$ and a relation $r$ in vectors $\vc{x_i} \in \realset^{d}$ and $\vc{r} \in \realset^d$ respectively, where
$d$ be the embedding size. 
In what follows, we let $f(x, p)$ be a function that computes the convolution of the embedding of entity~$x$ with the corresponding convolution filters associated to position $p$ and outputs a vector (see~\cite{HypE} for more details).
We let $\scoreH{}$ to be the HypE scoring function.

\begin{theorem}[HypE \op{Selection} 1 (a)]
\label{theorem:Hype selection1}
For arbitrary relation $s$
and arbitrary embeddings for entities $x_1,\dots,x_{n-1}$,
there exists \textbf{no} parametrization for relation $t$ in HypE such that
\begin{equation}\label{eq:selection-HypE}
    \scoreH(t(x_1,\dots,x_{n-1})) = \scoreH(s(x_1,\dots,x_{n-1}, x_{n-1}))\\
\end{equation}
\end{theorem}

\begin{proof}

To see why there is no parametrization for $t$ that satisfies 
Equation~\eqref{eq:selection-HypE}, we first expand the left and right hand side of the equality as follows. 

\begin{equation*}
    \scoreH(t(x_1,\dots,x_{n-1})) = \scoreH(s(x_1,\dots,x_{n-1}, x_{n-1}))\\
\end{equation*}
\begin{equation*}
    \begin{split}
    & \sum_{i=1}^{d} \vc{t}[i] \times f(x_1, 1)[i] \times \dots \times f(x_{n-1}, n-1)[i] = \\ 
    & \sum_{i=1}^{d} \vc{s}[i] \times f(x_1, 1)[i] \times \dots \times f(x_{n-1}, n-1)[i] \times f(x_{n-1}, n)[i]
    \end{split}
\end{equation*}

For this equation to hold, the following should hold for arbitrary entities $x_1,\dots,x_{n-1}$, and $\forall ~~ 1 \leq i \leq d$.

\begin{equation}\label{eq:t(x)}
or 
\begin{cases}
  f(x_1, 1)[i] \times \dots \times f(x_{n-1}, n-1)[i] = 0\\
  \vc{t}[i] = \vc{s}[i] \times f(x_{n-1}, n)[i]
\end{cases}
 \end{equation}

For an entity $x$ at position $p$ in a tuple, the output of function $f(x, p)$ depends on the embedding of entity $x$ and the convolution filters associated with position $p$. 
Note that these convolution filters are shared among all relations in the knowledge hypergraph and are not specific to relation $s$ or $t$. 

As we want Equation~\eqref{eq:t(x)} to hold for arbitrary entity embeddings, it is easy to see that there exists at least one setup for convolution filters and entity embeddings for which none of the factors in the product $f(x_i,1)[i] \times \dots \times f(x_{n-1}, n-1)[i]$ is zero.

Now consider such an embedding setting. 
The only way Equation~\eqref{eq:t(x)} is satisfied is when we set \mbox{$\vc{t}[i] = \vc{s}[i] \times f(x_{n-1}, n)[i]$} for all $1 \leq i \leq d$. 
Observe that by setting the embedding of relation $t$ to the embedding of $s$ times $f(x_{n-1}, n)$ for a particular entity $x_{n-1}$, we have effectively set it to a fixed value. 
Now consider an entity $x_k$
such that $f(x_k,n)\neq f(x_{n-1},n)$,
and apply the selection $t$ to $x_k$ by replacing $x_{n - 1}$ with $x_k$ in Equation~\eqref{eq:selection-HypE}. 
The equality condition in the equation will not hold, because none of the conditions in Equation~\eqref{eq:t(x)} hold. 
Therefore, in this setup, $t$ fails to represent $s$ for arbitrary entity embeddings.

Given that relation $t$ should represent \op{selection} of $s$
for all the entities in the hypergraph, setting it to a fixed value
will make it incapable of representing \op{selection} for arbitrary entities.
%
We can thus conclude that there is no parametrization for HypE such that it can represent \op{selection} for arbitrary embeddings of relations and entities.
\end{proof}

We now show that when the embedding size is less than the number of entities $|\entset|$ in the hypergraph, there is no setting for which HypE would be able to represent \op{selection} without losing full expressivity.

\begin{theorem}[HypE \op{Selection 1 (b)}]
\label{theorem:Hype selection2}
For arbitrary relation~$s$, there exists \textbf{no} parametrization for relation $t$ in HypE 
having scoring function $\scoreH$
such that for arbitrary entities $x_1,\dots,x_{n}$
and embedding dimension $d < |\entset|$, where $|\entset|$ is the number of entities in the knowledge hypergraph,
HypE remains fully expressive and 
\begin{equation*}
    \scoreH(t(x_1,\dots,x_{n-1})) = \scoreH(s(x_1,\dots,x_{n-1}, x_{n-1}))\\
\end{equation*}
\end{theorem}

\begin{proof}

Consider $|s| = 2$ and $|t| = 1$ and
tuples $t(x_j)$ and $s(x_j,x_j)$ such that relation $t$ is a \op{selection} of $s$.
We will show that when the embedding dimension $d$ of the entities is smaller than the number of entities in the knowledge hypergraph, there is no parametrization of HypE that gets $t$ to represent $s$ while retaining full expressivity.

For the sake of contradiction, assume that the lemma statement is false. Then there exists a parametrization for $t$ such that HypE is fully expressive and that
\begin{equation*}
    \scoreH(t(x_j)) = \scoreH(s(x_j, x_j))
\end{equation*}

Expanding the score function for $s$ and $t$ we get 
\begin{gather*}
    \sum_{i=1}^{d} \vc{t}[i] \times f(x_j, 1)[i] = \sum_{i=1}^{d} \vc{s}[i] \times f(x_j, 1)[i] \times f(x_{j}, 2)[i]
\end{gather*}
\begin{gather*}
    \Rightarrow
    \sum_{i=1}^{d} f(x_j, 1)[i] \times \left(\vc{t}[i] - \vc{s}[i] \times f(x_j, 2)[i]\right) = 0
\end{gather*}

For this equation to hold for arbitrary entities $x_1, x_2, \dots, x_n$,
it should hold for all $1 \leq i \leq d$ and all $1 \leq j \leq |\entset|$
%
%
for which $\vc{s}(x_j, x_j)$ is true. 
Assuming that $d < |\entset|$, we need to have the following hold for all $1 \leq j \leq |\entset|$ and $1 \leq i \leq d$.

\begin{equation}\label{eq:t(x1)}
or 
\begin{cases}
  f(x_j, 1)[i] = 0 \\
  \vc{t}[i] = \vc{s}[i] \times f(x_j, 2)[i] \\
\end{cases}
 \end{equation}

We claim that to satisfy the above equation simultaneously for all possible $i$ and $j$, 
there must be at least one entity $x_k$ for which the convolution function returns a zero vector; that is, $f(x_k,1)[i] = 0$ for all $0 \leq i \leq d$.
To see why this is true, assume the contrary; that is, for 
%
%
%
each convolution filter $f(x_j,1)$ has at least one bit that is different than zero for arbitrary entity $x_j$. 
Without loss of generality, let the $j$th bit $f(x_j,1)[j] = K_j$, 
for all $1 \leq j \leq |\entset|$ and $K_j \in \mathcal{R^*}$. 
Then, to satisfy Equation~\ref{eq:t(x1)} at index $j$, we must set $\vc{t}[j] = \vc{s}[j] \times f(x_j, 2)[j]$.
As Equation~\ref{eq:t(x1)} must be satisfied for all entities, then
all other entities $x_k$ with $k \neq j$ must have their $j$-th bit set to zero.
Thus $f(x_k, 1)[j] = 0$ for all $0 \leq j \leq d$ and $j \neq k$. See Figure~\ref{fig:selection}.
Since we have $d < |\entset|$,
and by the pigeon-hole principle, there must be at least one entity $x_{d+1}$
such that $f(x_{d+1}, 1)=0$ for all $1 \leq i \leq d$. 
This contradicts the original assumption, and thus 
there exists at least one $k$ for which $f(x_k, 1)$ returns a zero vector. 

This would further imply that any tuple having $x_{k}$ in the first position will 
have a score $\scoreH{}$ of zero.
This would be regardless of the relation in the tuple, or whether or not it is true. 
This violates the full-expressivity of HypE, thus contradicting the original assumption that the lemma statement is False.
Therefore, when $d < |\entset|$, there is no parametrization for $t$ such that HypE represents \op{selection} while retaining full expressivity.
\end{proof}

\begin{figure}[t]
\label{fig:selection}
\begin{center}
    \includegraphics[width=0.95\columnwidth, trim={0 0.0cm 0 0.0cm},clip]{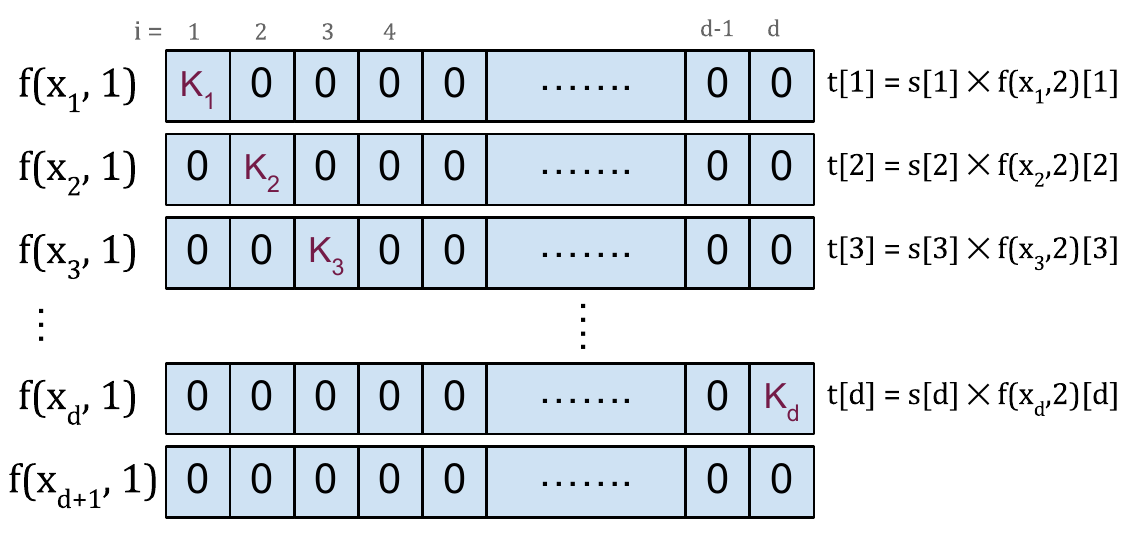}
  \end{center}
  \caption{Proof of Theorem~\ref{theorem:Hype selection2}: If at least one value of each vector $f(x_i,1)$ is different than zero ($K_i$ > 0) and the value of at most one position $i$ can be non-zero (the $d$ columns in the figure), then by the pigeon-hole principle, we will run out of indices as we have more entities in the knowledge hypergraph. In the above image, the entity $x_{d+1}.$ }
  \label{chart:ws}
\end{figure}

\section{Datasets}\label{appendix:datasets}

\subsection{Public Dataset Statistics}
Table~\ref{table:statistics} summarizes the statistics for the public knowledge hypergraph datasets \acr{JF17K}, \acr{FB-auto}, and \acr{m-FB15K}.

\begin{table*}[t]
\footnotesize
\setlength{\tabcolsep}{2pt}
\caption{Dataset Statistics.}
\label{table:statistics}
\begin{center}
\setlength{\tabcolsep}{3pt}
\begin{tabular}{l|ccccc|ccccccc}
\multicolumn{3}{c}{}
& \multicolumn{3}{c}{number of tuples}
& \multicolumn{5}{c}{number of tuples with respective arity}\\
\cmidrule(lr){4-6} \cmidrule(lr){7-11}
Dataset    & $|\mathcal{E}|$  & $|\mathcal{R}|$  & \#train  & \#valid & \#test & \#arity=2 & \#arity=3& \#arity=4& \#arity=5& \#arity=6 \\ \hline
\acr{JF17K} & 29,177 & 327 & 77,733 & -- & 24,915 & 56,322 & 34,550 & 9,509 & 2,230 & 37\\
\acr{FB-auto} & 3,410 & 8 & 6,778 & 2,255 & 2,180 &  3,786 & 0 & 215 & 7,212 & 0\\
\acr{m-FB15K} & 10,314 & 71 & 415,375 & 39,348 & 38,797 & 82,247 & 400,027 & 26 & 11,220& 0\\
\end{tabular}
\end{center}

\end{table*}

\subsection{Synthetic Dataset}
To study and evaluate the generalization power of hypergraph completion models in a controlled environment, we generate a synthetic dataset called \synthdata{} (RELational Erd\H{o}s-R\'enyi). To create \synthdata, we build on the Erd\H{o}s-R\'enyi model~\cite{erdos1959} for generating random graphs and extend it to knowledge hypergraphs. In what follows, we discuss the details of our algorithm.

\subsubsection{Knowledge Hypergraphs}
A \emph{knowledge hypergraph} is a directed hypergraph $H = (V, E, R)$ with nodes (entities)~$V$, edges (tuples)~$E$, and
edge labels (relations)~$R$ such that:

\begin{itemize}
    \item every edge in the hypergraph consists of an \emph{ordered} sequence of nodes, 
    \item every edge has a label $r_i \in R$, and
    \item edges with the same label are defined on the same number of nodes. 
\end{itemize}

Observe that in knowledge hypergraphs, edges having the same label form a uniform directed hypergraph (all edges defined on the same number of nodes). We can thus think of $H$ as the combination of $|R|$ directed uniform hypergraphs.

\subsubsection{Extending Erd\H{o}s-R\'enyi to Knowledge Hypergraphs}
In the Erd\H{o}s-R\'enyi model, all graphs with a fixed number of nodes and edges are equally likely. 
Equivalently, in such a random graph, each edge is present in the graph with a fixed probability $p$, independent of other edges. 
In this section, we describe a method of generating a random \emph{knowledge hypergraph} inspired by the 
Erd\H{o}s-R\'enyi process.

Let $n$ be the (predefined) number of nodes in the hypergraph and $n_r$ be the number of relations. 
We let $R$ be a list of relations defined in terms of arity and a probability that influences the number
of tuples generated for that given relation. More formally,
$$\displaystyle R = \{(k_i, p_i) | k_i = arity, 0 \leq p_i \leq 1, \forall i = 0,\dots,n_r\}$$

The expected number of edges generated for a given relation $r_i$ is $m_i = k_i!{n \choose k_i} p_i$.
As the process of including edges in the graph is a Binomial, we can compute this expected
value by sampling from the following probability density.
$$\displaystyle P(m_i) = {N \choose {m_i}}p_i (1-p_i)^{N-m_i}  $$
where $N = k_i! {n \choose k_i}$ is the number of possible $k_i$-uniform (directed) edges in the hypergraph.

The running time of Algorithm~\ref{erdos-renyi} depends on the number of nodes $n$ and the arity $k_i$ of each relation.
Let $\displaystyle k = max_{k_i \in R} k_i$.  Thus the running time of Algorithm~\ref{erdos-renyi} is $O(|R|n^k)$.
\vspace{10.5pt}

\begin{algorithm}[H]
\caption{generate\_knowledge\_hypergraph($V$, $R$))}
\begin{algorithmic}
  \label{erdos-renyi}
  \STATE edge\_list = []
  \STATE $n$ = len($V$)
  \FOR{$r$, ($k$, $p$) in \emph{enumerate}(R)}
    \STATE $N = k!{n \choose k}$
    \STATE $m$ = random.binomial($N$, $p$)
    \COMMENT{result of flipping a coin $N$ times with probability of success $p$}
    \STATE edge\_count = 0
    \WHILE{edge\_count $<= m$}
      \STATE edge = random.sample($V$, $k$)
      \COMMENT{select $k$ vertices from $V$ at random}
      \IF {edge not in edge\_list}
        \STATE edge\_list.append([$r$] + edge)
        \STATE edge\_count = edge\_count  + 1
      \ENDIF
  \ENDWHILE
\ENDFOR
\STATE return edge\_list
\end{algorithmic}
\end{algorithm}

\subsubsection{Dataset Generation}
To evaluate a model on how well it represents relational algebra operations, we generate 
a set of ground-truth true tuples, each of which is the result of repeated application of a primary relational algebra
operation to an existing tuple (hyperedge). The operations we are interested in are \op{renaming}, \op{projection}, \op{selection}, \op{set union} and \op{set difference}.

\vspace{10.5pt}
\begin{algorithm}
\label{alg:ground truth}
\begin{algorithmic}
\caption{\mbox{generate\_ground\_truth($V$, $R$, $n\_derived\_tuples$)}}
\STATE $E$ = generate_knowledge_hypergraph($V$, $R$)
\FOR{$i$ in $range(n\_derived\_tuples)$}
\STATE $op$ = randomly select one primary operation
\STATE $tuple$ = randomly select one hyperedge from $E$
\STATE apply $op$ to $tuple$
\STATE add $tuple$ to the set of edges $E$
\ENDFOR
\end{algorithmic}
\end{algorithm}

\vspace{10.5pt}

Finally, the complete algorithm to generate the train, valid, and test sets of the synthetic dataset is described in Algorithm~\ref{alg:data-generator} below.

\vspace{10.5pt}
\begin{algorithm}
\caption{synthesize\_dataset($V$, $R$, $n\_derived\_tuples$)}
\begin{algorithmic}
\label{alg:data-generator}
\STATE \mbox{$ground\_truth$ = generate\_ground\_truth($V$, $R$,} \indent 

~~~~~~~~~~~~~~~~~~~~~~~~~\mbox{$n\_derived\_tuples$)}
\STATE $relational\_data$ = sub-sample from $ground\_truth$
\STATE train, valid, test = randomly split  $relational\_data$ into train, valid and test
\end{algorithmic}
\end{algorithm}

\section{Implementation Details}\label{appndix:implementation}
We implement \ourmodel{} in PyTorch~\citep{pytorch} and use Adagrad~\citep{adagrad} as the optimizer and dropout~\citep{srivastava2014dropout} to regularize the model. 
We perform early stopping and hyperparameter tuning based on the MRR on the validation set. We fix the maximum number of epochs to $1000$ and batch size to $128$. We set the embedding size and negative ratio to $200$ and $10$ respectively. 
We tune $lr$ (learning rate) and $w$ (window size) using the sets $\{0.05, 0.08, 0.1, 0.2\}$, and $\{1, 2, 4, 5, 8\}$ (first five divisors of 200).
We tune $\sigma$ (nonlinear function) using the set $\{tanh, sigmoid, exponent\}$ for the \acr{JF17K} dataset. The results of the different nonlinear functions are in Table~\ref{table:sigma results}. As $sigmoid$ outperforms the $tanh$ and $exponent$, we only tried $sigmoid$ for other datasets.

Reported results for the baselines on \acr{JF17K}, \acr{FB-AUTO}, and \acr{M-FB15K} are taken from the original paper except for that of GETD~\cite{GETD}. The original paper of GETD only reports results for arity 3 and 4 as trained and tested separately in the corresponding arity. 
However, in our experimental setup, we train and test in a dataset containing relations of different arities. For that, we train and test GETD. As GETD learns a tensor of dimension $|r|$ for each relation $r$, it needs $d^{|r|}$ (with $d$ as embedding size) number of parameters. The original paper proposes smart strategies to reduce the number of parameters to be learned by the model. However, we still need to store the relation embedding and thus need to store $d^{|r|}$ floating-point numbers for each relation $r$. Because of our memory limitation ($16$GB GPU), we could only train the GETD model for embedding size of less than $10$.

The sign ``-'' in Table 1 indicates that the corresponding paper has not provided the results.

For the experiment on the synthetic dataset, we compare our model with m-TransH~\cite{m-TransH} and HypE~\cite{HypE}, which are the only competitive baselines that have provided the code (\url{https://github.com/ElementAI/HypE}, \url{https://github.com/wenjf/multi-relational_learning}).

The code and the data for all the experiments is
available at \url{https://github.com/baharefatemi/ReAlE}.

\begin{table}
\footnotesize
\setlength{\tabcolsep}{2pt}
\caption{Knowledge hypergraph completion results for \ourmodel{{} }on \acr{JF17K} for different $\sigma$ (nonlinear function).
}
\label{table:sigma results}
\begin{center}
\setlength{\tabcolsep}{3pt}
\begin{tabular}{l|cccc}
\multicolumn{1}{c}{}
& \multicolumn{4}{c}{\acr{JF17K}}\\
\cmidrule(lr){2-5} 
Model    & MRR   & Hit@1 & Hit@3 & Hit@10\\\hline
\ourmodel{} ($\sigma$ as exponent)  (Ours) & 0.394 & 0.311 & 0.428 & 0.548\\
\ourmodel{} ($\sigma$ as tanh)  (Ours) & 0.512 & 0.430 & 0.548 & 0.667\\
\ourmodel{} ($\sigma$ as sigmoid)  (Ours) & \textbf{0.530} & \textbf{0.454} & \textbf{0.563} & \textbf{0.677}\\
\end{tabular}
\end{center}
\end{table}

\end{document}